\documentclass[twoside]{article}

\usepackage[accepted]{aistats2026}
\usepackage[round]{natbib}

\usepackage{proj}

\begin{document}

%

%

\twocolumn[

\aistatstitle{Incoherence in Goal-Conditioned Autoregressive Models}

\aistatsauthor{ Jacek Karwowski \And Raymond Douglas }

\aistatsaddress{ Department of Computer Science \\ University of Oxford \And Telic Research} ]

\begin{abstract}
  We investigate mathematically the notion of \emph{incoherence}: a structural issue with reinforcement learning policies derived by naive goal-conditioning of autoregressive models. We focus on the process of re-training models on their own actions, that is, fine-tuning offline-learned policies with online RL. We prove that it decreases incoherence and leads to an improvement in return, and we aim to characterise the resulting trajectory of policies. By re-framing standard notions of control-as-inference and soft Q learning, we establish a three-way correspondence with two other ways of understanding the iterative re-training process: as \emph{folding the posterior into the reward} and, in the deterministic case, as \emph{decreasing the temperature parameter}; the correspondence has computational content via the training-inference trade-off. Through soft-conditioning generative models, we discuss the link between incoherence and the \emph{effective horizon} of \cite{laidlaw_bridging_2024}.
\end{abstract}

\section{INTRODUCTION}\label{sec:introduction}

\begin{figure}[t]
    \centering
    \begin{tikzpicture}[thick, scale=1]
      \node (root) at (0,0) {$\emptyset$};
      \node (R) at (2,1) {$\scalebox{2}{\mountain}$};
      \node (Rc) at (2,-1) {$\scalebox{2}{\forest}$};
      \node (RR) at (4,1.5) {$\scalebox{2}{\gold}$};
      \node (RRc) at (4,0.5) {$\scalebox{2}{\skull}$};
      \node (RcR) at (4,-0.5) {$\scalebox{2}{\silver}$};
      \node (RcRc) at (4,-1.5) {$\scalebox{2}{\silver}$};
    
      \draw[->] (root) -- (R) node [midway, above left] {$\frac{1}{2}$};
      \draw[->] (root) -- (Rc) node [midway, below left] {$\frac{1}{2}$};
      \draw[->] (R) -- (RR) node [midway, above left] {$\frac{1}{2}$};
      \draw[->] (R) -- (RRc) node [midway, below left] {$\frac{1}{2}$};
      \draw[->] (Rc) -- (RcR) node [midway, above left] {$\frac{1}{2}$};
      \draw[->] (Rc) -- (RcRc) node [midway, below left] {$\frac{1}{2}$};
    
      \node at (6,1.5) {$\Prob{R = 1|\gold} = 1$};
      \node at (6,0.5) {$\Prob{R = 1|\skull} = 0$};
      \node at (6,-0.5) {$\Prob{R = 1|\silver} = \frac{3}{4}$};
      \node at (6,-1.5) {$\Prob{R = 1|\silver} = \frac{3}{4}$};
    \end{tikzpicture}
    \caption{Tree representation of the MDP defining the \emph{mountain race} (\Cref{example:two-cards}). States are represented as the tree nodes, we put a uniform prior over actions (depicted as arrows $\nearrow, \searrow$). Rewards for each terminal state (tree leaf) are written on the right.}
    \label{fig:two-cards}
\end{figure}

Control-as-inference reframes reinforcement learning as an inference problem: instead of explicitly trying to search for an optimal policy in a given environment, one first constructs a generative model over actions or trajectories, and then conditions it on the goal, deriving the policy from the posterior. Doing so allows for more abstract characterisation of the resulting policies, without reference to the internals of the particular learning algorithm.

In this work, we focus on a multi-step environment, where the derived policy is used autoregressively: in each time-step, the generative model is conditioned on both the fixed goal and the current state. In this setup, we refine and characterize the recently introduced notion of \emph{predictor-policy incoherence}~\citep{douglas2024limitations}, or simply \emph{incoherence}. This is a \emph{structural} problem with such policies, which is not fixed by improving predictive accuracy of the underlying model. It stems from the fact that, given a binary reward $R$ and a prior over actions $p(a|s)$, the conditioned policy $\pi(a|s) = p(a|s, R = 1)$ is an answer to the question: 
\begin{quote}
    \emph{(1) Which action to take in state $s$, such that, if later choices are made according to the prior $p$, the outcome will lead to R?}
\end{quote}
and \emph{not} the question:
\begin{quote}
    \emph{(2) Which action to take in state $s$, such that, if later choices are sampled auto-regressively from $\pi$, the outcome will lead to R?}
\end{quote}
Let us illustrate this problem by looking at a simple deterministic Markov decision process.

\begin{example}[Mountain race]\label{example:two-cards}
An agent is racing in the mountains. Starting in the state $\emptyset$, it is given a choice between two trails: a slower path down through the forest $\forest$, and a faster path up over the ridge $\mountain$. Both trails fork in the middle: following the path up on the $\mountain$ junction leads quickly to the finish line $\gold$, while the path down ends in a chasm $\skull$. On the $\forest$ path, both choices lead, albeit more slowly, to the finish line $\silver$. Full game tree is presented in~\Cref{fig:two-cards};

The agent is given binary reward $R$ with probability $1$ if finished $\gold$, with probability $\frac{3}{4}$ if finished $\silver$, and with probability $0$ if finished $\skull$.  Below we denote actions as $\{\nearrow, \searrow\}$. The game is deterministic w.r.t. player's choice, but for the purpose of computing control-as-inference solution, we assume a uniform prior over actions.
\end{example}

We might intuitively understand this game in the following way: an agent is given a choice between ``risky" play of $\mountain$, which gives it a large payoff of $1$ -- but only if it chooses the trail correctly later on, and a ``safe" play of $\forest$, which gives it a smaller but more certain payoff, with probability $\frac{3}{4}$.

We compute the policy given by conditioning on the outcome $R = 1$ as:
\begin{align*}
    \pi(a|\emptyset) = \pr{a|\emptyset, R=1} &=
    \frac{\pr{R=1|\emptyset, a}\pr{a|\emptyset}}{\pr{R=1|\emptyset}}\ \\ &\ \propto\ \pr{R=1|\emptyset, a}
\end{align*}
Since the transition dynamics are deterministic, we plug this into the formulas for $\pi$ and get:
\begin{align*}
    \pi(\nearrow|\emptyset) &\ \propto\ \pr{R=1|\emptyset, \nearrow} =  \pr{R=1|\mountain} = \frac{1}{2}\\
    \pi(\searrow|\emptyset) &\ \propto\ \pr{R=1|\emptyset, \searrow} =  \pr{R=1|\forest} = \frac{3}{4}
\end{align*}

We now observe that the RHS evaluates simply to the prior probability of attaining $R$. In other words, it does not take into account that the policy $\pi$ is to be used autoregressively. On the other hand, a coherent policy $\hat\pi$ should have conditioned on the fact that $\hat\pi(R = 1 | \mountain) = 1$, since, arriving in the state $\mountain$, the reward-conditioned policy is also used.

As the above example shows, a naive autoregressive control-as-inference approach might result in incoherent policies. Intuitively, we understand the incoherence as the incapacity of the agent to anticipate its own future actions, in line with the difference between Questions (1) and (2). This also suggests a possible, and indeed, widely-used, fix: to fine-tune the agent on its own actions. Understanding the dynamics of this process, that is, what kinds of policies does it produce along the re-training trajectory, as well as relating it to incoherence, is thus the second point of this work.

\subsection{Contributions and outline}
The primary contribution of this paper is to properly (re-)define incoherence, and then reframe and unify existing ways of \emph{tightening} goal‑conditioned autoregressive policies in that framework. We discuss connections and parallels to many well-known approaches to soft RL in \Cref{sec:related-work}. After preliminaries in~\Cref{sec:preliminaries}, \Cref{sec:incoherence} concerns the issues of incoherence; we refine and generalise the definition given previously by~\citet{douglas2024limitations} to better account for the difference between \emph{coherence} and \emph{optimality}, and describe some properties of coherent policies. In \Cref{sec:removing-incoherence}, we discuss three ways of updating RL policies to remove incoherence: 
\begin{itemize}
    \item In~\Cref{def:goal-conditioning}, as fine-tuning (re-conditioning) policies on their own trajectories, implemented e.g. as collecting trajectories from the model acting in an environment and re-training the model based on the augmented dataset.
    \item In~\Cref{def:temperature} as decreasing the temperature parameter, which can also be understood as manipulating the strength of the entropy regularisation in KL-regularised RL, implemented using e.g. inference-time best-of-$n$ rejection sampling in practice.
    \item In~\Cref{def:folded,def:folded-orig}, as folding the posterior over actions into the reward, which resembles the trick of disregarding a prior by folding it into the reward~\citep{levine_reinforcement_2018}. Since the posterior depends on the reward, the process has to be repeated iteratively until convergence. This technically modifies the MDP, instead of modifying the policy. Even if the starting reward was only given sparsely in end-state, this process distributes it around the MDP.
\end{itemize}
The connections between those perspectives allow us to transfer properties between those formulations, for example, to provide a rate of convergence to the optimal policy of the re-conditioning. We discuss some related perspectives, in particular a connection to effective horizon~\cite{laidlaw_bridging_2024}, and limitations, in~\Cref{sec:effective}. All proofs are given in~\Cref{appendix:proofs}. We also give a code appendix implementing tabular MDPs experiments and validating our main results numerically at \href{https://github.com/jkarwowski/incoherence}{https://github.com/jkarwowski/incoherence}.

\section{RELATED WORK}\label{sec:related-work}

\textbf{Control as Inference} view, where optimality is modeled via a binary variable $\mathcal O$ with likelihood $p(\mathcal O| s,a)\ \propto\ e^{\alpha r(s,a)}$, is a central influence on our work. In that setup, deterministic dynamics allow for exact inference; in stochastic dynamics the policy solves a variational problem~\citep{levine_reinforcement_2018}. We emphasize that \emph{coherence} requires computing expectations under the posterior policy, not under a fixed prior; which forces the iterative posterior‑folding we analyse here. \cite{odonoghue_making_2020} focus on a similarly-termed incoherence in RL-as-inference, having to do with the fact that posterior probabilities do not reflect epistemic uncertainty about actions, and the effect this has on the exploration-exploitation trade-off. The connection between control and inference has been originally studied by~\citet{todorov_linearly-solvable_2006}, and in the context of Inverse RL by~\citet{ziebart_maximum_2008,gleave_primer_2022}.

\textbf{$\KL$‑regularized policy search methods} optimise expected return with a $\KL$ trust region to a prior, yielding E‑step that is a Boltzmann distribution over Q and an M‑step that updates the actor~\citep{peters_relative_2010,schulman_proximal_2017,abdolmaleki_maximum_2018}. Our ``temperature" view is the Lagrange multiplier of the $\KL$ constraint; our iterative posterior‑folding viewpoint recovers the same policy sequence when dynamics are deterministic. We thus reframe $\KL$‑regularized policy improvement as \emph{restoring coherence} for goal‑conditioned policies, as well as extend it to the stochastic or multi-step environment as compared to~\citet{korbak_rl_2022}.

\textbf{RvS~\citep{brandfonbrener_when_2023}} as well as Upside-down RL~\citep{srivastava_training_2021} and Decision Transformer~\citep{chen_decision_2021} approaches condition actions on desired returns and act autoregressively. Theory and experiments show that naive conditioning yields systematic failures in stochastic environments (trajectory ``luck”), and that separating controllable from uncontrollable randomness improves behavior \cite{strupl_upside-down_2022,paster_you_2022,yang_large_2024}. Our notion of incoherence is a structural account of the same issue: the posterior used for conditioning assumes futures incompatible with the deployment policy; our equivalence results characterize procedures that realign them.

\textbf{Expert iteration and MCTS} are a widely used and successful method of improving models' performance. Well-known applications include AlphaZero~\citep{silver_mastering_2017, silver_general_2018} and MuZero~\citep{schrittwieser_mastering_2020}, algorithms using Monte Carlo Tree Search~\citep{browne_survey_2012}. One treatment of those kinds of algorithms combining search and improvement of the policy was proposed by~\citep{anthony_thinking_2017} under the name of Expert Iteration. We focus on the abstract properties of the re-training process, in a situation of soft-conditioning policies, a combination which prior work did not address. Self-play in games~\citep{macleod_game_2005, openai_dota_2019, vinyals_grandmaster_2019} is another related strategy of improving the performance of a policy, however, it is distinct from the model learning about its own future policy in a non-competitive setup we study here.

\textbf{Large language models} are capable of general world modelling~\citep{radford_language_2019, brown_language_2020, bai_constitutional_2022, touvron_llama_2023}, and thus capable of simulating agents~\citep{shanahan_role-play_2023, douglas2024limitations}. Although our setup here considers a model trained on a single environment,~\citet{andreas_language_2022} argues that this perspective applies to LLMs whose training corpus comes from human actions in the internet environment. Prior work on eliciting agents through prompting and scaffolding methods~\citep{significant_gravitas_autogpt_2024, yang_auto-gpt_2023} conditions the base model in a purely formal sense. The exact nature of this form of conditioning, as well as its connection to RL fine-tuning methods (such as RLHF~\citep{christiano_deep_2017}, DPO~\citep{rafailov_direct_2023}, GRPO~\citep{shao_deepseekmath_2024}) is an open problem, such as e.g.~\emph{RLHF Conditioning Hypothesis}~\citet{hubinger_conditioning_2023}. It has been argued that the next token prediction objective alone encourages local consistency but not global planning~\citep{mccoy_embers_2023}. Recent work on \textsc{Coconut}~\cite[Section 5.1, Fig. 7]{hao_training_2024} shows that the answer‑token distribution along a latent reasoning tree acts like an implicit value function, an empirical point of support for our energy/value interpretation of goal‑conditioning.

\textbf{Residual Energy‑Based Models} for text generation add residual energy atop a base autoregressive model to steer sequence probabilities~\citep{deng_residual_2020}. Our folding-posterior‑into‑reward iteration is the control analogue of adding residual energy $\log p_\pi(a|s)$ to the base reward, with dynamics then re‑evaluated under the new energy. This clarifies when temperature annealing can substitute for residual terms (deterministic dynamics) and when it cannot (stochastic). Energy-based policies of the general form $\pi(a|s)\ \propto\ \exp(E(s, a))$ had been studied by~\citet{haarnoja_reinforcement_2017}, and in the context of off-policy RL used to develop Soft Actor-Critic algorithm~\citet{haarnoja_soft_2018}. This line of work is focused on the iterative approach using Bellman updates and countering distributional shift, while we look at the re-training as making policy internally consistent (still requiring computational effort).

\section{PRELIMINARIES}\label{sec:preliminaries}

A \emph{Markov Decision Process} (MDP) with a time horizon $T \in \mathbb{N}$ is a tuple $\FullMDP$, where $\S$ is a set of \emph{states}, $\A$ is a set of \emph{actions}, $\T: \SxA \to \dist{S}$ is a transition function, $\m \in \dist{S}$ is the initial distribution over states, $r: \SxA \to [-\infty, 0]$ is the reward function (assumed to be non-positive), and $\gamma \in [0, 1]$ is a time discount factor. We will assume $\gamma = 1$ without loss of generality (as one can always convert an MDP with $\gamma < 1$ to $\gamma = 1$ by introducing an auxiliary terminal state). A \emph{trajectory} is a sequence $\xi = (s_0, a_0, s_1, \ldots, s_T, a_T)$ such that $a_i \in \A$, $s_i \in \S$ for all $i$. A \emph{policy} is a function $\policy :\S \to \dist{A}$. We say that the policy $\policy$ is deterministic if for each state $s$ there is some $a \in \A$ such that $\policy(s) = \delta_{a}$. Each policy $\policy$ on an MDP induces a probability distribution over trajectories $p_\pi(\xi)$; drawing a trajectory $(s_0, a_0 \ldots, s_T, a_T)$ from a policy $\policy$ means that $s_0$ is drawn from $\m$, each $a_i$ is drawn from $\policy(a_i|s_i)$, and $s_{i+1}$ is drawn from $\T(s_{i+1}|a_i,s_i)$ for each $i$. For a policy $\pi(a|s)$ and a reward $r(s,a)$ we define the return $J(\pi)$ to be $J(\pi) = \Expect{s_t, a_t \sim \pi}{\sum_{i=0}^{T} r(s_i, a_i)}$. 
We will sometimes distinguish a policy $\pi(a|s)$ and a prior over actions $p(a|s)$: even though these have the same type, they play different conceptual roles, with the policy being subject to the optimisation process and thus not fixed.

\section{THE THEORY OF INCOHERENCE}\label{sec:incoherence}

To build a quantitative measure of incoherence, we first need to define soft $Q$ and $V$ functions, which are better-suited for working in a probabilistic setup than the standard definitions.

\begin{definition}[Soft $Q$ and $V$ functions]\label{def:softvq}
    For a policy $\pi(a_t|s_t)$, a reward function $r(s_t,a_t)$ and transition dynamics $\tau(s_{t+1}|s_t, a_t)$, the soft $V$ and $Q$ functions are defined by mutual recursion as:
    \begin{align*}
    Q^\pi(a_t, s_t) &= r(s_t,a_t) + \log \Expect{s_{t+1} \sim \tau(s_t,a_t)}{\exp{V^\pi(s_{t+1})}} \\
    V^\pi(s_t) &= \log \Expect{a_t \sim \pi(a_t|s_t)}{\exp{Q^\pi(s_t,a_t)}}
    \end{align*}
\end{definition}

We note that our definition is parametrised by a policy $\pi$, and thus slightly different from the one given in~\citep{haarnoja_reinforcement_2017, levine_reinforcement_2018}: the difference is in how we compute the expectation in $V^\pi$. Instead of drawing the action from a fixed prior, which is then routinely replaced w.l.o.g. by a uniform distribution over $A$, we draw it according to the policy $\pi$. The point of this subtlety will become apparent in the next section. In any case, above definition gives us the following alternative characterisation.

\begin{proposition}[Characterisation of $V$ and $Q$]\label{prop:charqv}
    For any prior $\pi(a|s)$ and any non-positive reward function $r(a,s)$, we have simple expressions for the soft $Q$ and $V$ functions given by:
    \begin{align*}
        Q^\pi(a_t, s_t) &= \log p_\pi(\OO_{t:T} = 1|s_t,a_t) \\
        V^\pi(s_t) &= \log p_\pi(\OO_{t:T} = 1|s_t)
    \end{align*}
    where we define the auxiliary \emph{optimality variables} $\OO_t$ to have Bernoulli distributions with:
    \[
    p_\pi(\OO_t = 1|s_t,a_t) = e^{r(s_t,a_t)}
    \]
\end{proposition}

The above characterisation, making use of the auxiliary optimality variables $\OO_t$, showcases an approach that we will be utilizing throughout the rest of this paper. To put it differently, one might look at the (non-positive) reward $r(s, a)$ as the probability that taking action $a$ in the state $s$ is \emph{correct} - correctness constituting a latent property, as a more convenient way to cast MDP reward dynamics in a probabilistic setup.

We also have an alternative characterisation. Having a probability distribution over trajectories:
\[
p(\xi) = \prod_{t = 1}^{T} p(a_t|s_t)\tau(s_{t+1}|a_{t}, s_{t})
\]
as in~\Cref{sec:preliminaries}, we derive the goal-conditioned:
\[
p(\xi|\OO_{1:T}) = \frac{p(\OO_{1:T}|\xi)p(\xi)}{p(\OO_{1:T})}
\]
On the other hand, we have the distribution over trajectories given by locally goal-conditioning the prior over actions on the future reward:
\[
\hat{p}(\xi) = \prod_{t = 1}^{T} p(a_t|s_t, \OO_{t:T})\tau(s_{t+1}|a_{t}, s_{t})
\]
\begin{proposition}
    In case of deterministic dynamics, we have that $p(\xi|\OO_{1:T}) = \hat{p}(\xi)$.
\end{proposition}

Following~\citet{levine_reinforcement_2018}, we highlight a potential issue here. Since the policy derived by control-as-inference conditions the prior \emph{trajectories} distribution, (given by $(p, \tau)$ jointly) on the goal $\OO$, it results in an over-optimistic estimation of the environment dynamics, since it also conditions the stochastic dynamics. This will be the source of various issues separating deterministic and non-deterministic $\tau$ we encounter in next sections.

\subsection{Incoherence}
Incoherence is, in a sense, a statement about a policy's failure to act according to \emph{its} future roll-outs, and relying on the prior instead, as in the~\Cref{example:two-cards}. Thus, we judge a policy more \emph{coherent} the more its future returns influence the present decision in a consistent way. If we do not want to \emph{a priori} prescribe the exact way of the influence is exerted, we can leave the function as a parameter in the definition below.
\begin{definition}[Order‑respecting $f$]
    We say that a function $f$ is order-respecting if for all $x$ and $i\neq j$ we have that $f_i(x)$ is non‑decreasing in $x_i$ holding other coordinates fixed; $f_i(x)$ is non‑increasing in $x_j$ holding the rest fixed; and if $x_i\ge x_j$ then $f_i(x)\ge f_j(x)$.
\end{definition}
Examples of order-respecting $f$ include softmax with any temperature and the $\arg\max$ indicator (with discontinuities at ties).

\begin{definition}[$f$-soft Q policy]
    For an order-respecting function $f$ and a policy $\pi(a|s)$ its $Q, f$-soft policy $\pi^{Q,f}(a|s)$ is the probability distribution:
    \[
    \pi^{Q, f}(\cdot|s)\ \propto\ f(Q^\pi(s, \cdot))
    \]
\end{definition}
We note that because our MDP setup allowed rewards with zero-probability $p(\mathcal{O}_t = 1| s_t, a_t) = 0$,  we allow the rewards to be $-\infty$. We also note that in~\Cref{example:two-cards} (as well as in examples down below) for clarity of presentation, we have used reward $R$ to mean $\mathcal{O}_{T=2}$.

The incoherence is then the KL divergence between the current policy distribution, and the distribution prescribed by future returns.

\begin{definition}[$f$-incoherence]
    The incoherence of a policy $\pi(a|s)$ with respect to its $f$-soft Q policy is defined as the $\KL$ divergence over trajectories $\xi^\pi$:
    \[
    \kappa_f(\pi) = \KL_\xi(\pi(\xi) || \pi^{Q, f}(\xi))
    \]
\end{definition}

We note that by a well-known correspondence (see e.g.~\cite{haarnoja_soft_2018} or~\cite{belousov_kl_2017} for an informal write up) the $\KL$ divergence over trajectories can be rewritten as the per-state $\KL$ averaged over occupancy measure.
\begin{proposition}\label{prop:alternative-incoherence}
    If $d^t_\pi$ is the marginalised occupancy measure ${d^t_\pi(s_t) = p_{\xi \sim \pi}(S_t = s_t)}$, we have:
    \[
    \kappa_f(\pi) = \sum_{t = 1}^T \mathbb{E}_{s_t \sim d^t_\pi} \KL(\pi(\cdot|s_t) || \pi^{Q, f}(\cdot|s_t))
    \]
\end{proposition}

We say that a policy is $f$-\emph{coherent}, if $\kappa_f(\pi) = 0$. In particular, coherence does not mean optimality in the usual sense of maximizing the return. Deterministic policies are $f$-coherent for a single functional $f$.
\begin{proposition}
    Given deterministic dynamics $\tau$ and policy $\pi(a|s)$ and an order-respecting $f: \mathbb{R}^{|A|} \to \Delta(|A|)$, such that for $x$ having a unique maximum, $f(x)$ is an $\arg\max$ indicator, $\pi$ is $f$-coherent if and only if it is greedy w.r.t. its own soft‑$Q$ function.
\end{proposition}

We also note that from~\Cref{prop:charqv} we know that:
\[
Q^\pi(a_t, s_t) = \log p_\pi(\mathcal{O}_{t:T} = 1)
\]
so in this case, maximising $J(\pi)$ is equivalent to maximising the probability of optimality $\mathcal{O}_{1:T} = 1$.

What about the general case of stochastic dynamics? First, we show that we can construct the coherent policy iteratively, starting from any prior $p(a|s)$.

\begin{definition}[Iterated $f$-coherence]\label{def:iterated-coherence}
    Given a prior $\pi(a|s)$ and a function $f$, we define a sequence of policies $\pi^\BB_i$ recursively by:
    \begin{align*}
        \pi^\BB_0(a|s) = \pi(a|s) \qquad
        \pi^\BB_{i+1}(a|s) = \pi^{Q, f}(\pi^\BB_i(a|s))
    \end{align*}
\end{definition}
This procedure of iterated $f$-coherence $\pi^\BB$ converges to the $f$-coherent policy after at most $T$ steps.
It is the soft value iteration algorithm using an $f$-transformed soft $Q$-value.
\begin{proposition}
    The policy $\pi^\BB_T$ is $f$-coherent.
\end{proposition}

The coherence is defined up to the choice of the particular function $f$, which dictates how strongly future returns influence the behaviour in a particular step. We will be interested in a special class of coherent policies, given by Boltzmann distributions.

\begin{definition}[Boltzmann rationality]
    Given a policy $\pi$, the Boltzmann-rational policy $\pi^\delta$ with the parameter $\delta > 0$ is given by:
    \[
    \pi^\delta(a|s) \propto\exp\left({ \frac{1}{\delta} Q^\pi(a,s)}\right)
    \]
\end{definition}
It is known that in a single-step MDP setup, Boltzmann-rational policies are the solution to the satisficing maximum entropy problem (see, e.g. \citet[Appendix A]{jeon_reward-rational_2020}). We will not need this perspective here, but we return to it in~\Cref{sec:equivalence}.

\begin{figure}
        \centering
        \begin{tikzpicture}
          \node (root) at (0,0) {${\emptyset}$};
          \node (R) at (2,1) {$\scalebox{2}{\mountain}$};
          \node (Rc) at (2,-1) {$\scalebox{2}{\forest}$};
          \node (RR) at (4,1.5) {$\scalebox{2}{\gold}$};
          \node (RRc) at (4,0.5) {$\scalebox{2}{\skull}$};
          \node (RcR) at (4,-0.5) {$\scalebox{2}{\silver}$};
          \node (RcRc) at (4,-1.5) {$\scalebox{2}{\silver}$};
        
          \draw[->] (root) -- (R) node [midway, above left] {$\frac{4}{7}$};
          \draw[->] (root) -- (Rc) node [midway, below left] {$\frac{3}{7}$};
          \draw[->] (R) -- (RR) node [midway, above left] {$1$};
          \draw[->] (R) -- (RRc) node [midway, below left] {$0$};
          \draw[->] (Rc) -- (RcR) node [midway, above left] {$\frac{1}{2}$};
          \draw[->] (Rc) -- (RcRc) node [midway, below left] {$\frac{1}{2}$};
        
          \node at (6,1.5) {$\Prob{R} = \frac{1}{4}$};
          \node at (6,0.5) {$\Prob{R} = 0$};
          \node at (6,-0.5) {$\Prob{R} = \frac{3}{4}$};
          \node at (6,-1.5) {$\Prob{R} = \frac{3}{4}$};
        \end{tikzpicture}
        \caption{The fixed point of iterated $f$-coherence achieved after $T = 2$ iterations.}
        \label{fig:iterated-fixpoint}
    \end{figure}
\begin{example}[Boltzmann-coherent mountain race]
    We revisit the~\Cref{example:two-cards}. We might compute the Boltzmann-coherent policy for $\delta = 1$ by iterating the $f$-coherence from~\Cref{def:iterated-coherence}. After $T = 2$ iterations, we arrive at the fixpoint in~\Cref{fig:iterated-fixpoint}.
\end{example}
\begin{definition}[Boltzmann incoherence]\label{def:boltzmann-incoherence}
    Given a policy $\pi$ and a parameter $\delta > 0$, denoting $g(x) = \exp(x/\delta)$, we define the Boltzmann incoherence (or simply \emph{incoherence}) as $\kappa_\delta(\pi) = \kappa_{g}(\pi)$.
\end{definition}
Unrolling the definition, we have $\kappa_\delta(\pi) = \KL(\pi(a|s)||\pi^\delta(a|s))$.
In other words, a policy $\pi$ is Boltzmann-coherent with parameter $\delta$, if it is coherent with respect to its $f$-soft Q policy for $g$, i.e. softmax with the parameter $\delta$.

\begin{proposition}\label{lemma:optimal-maxent}
    Given any prior $\pi(a|s)$, there exists a policy $\pi^*(a|s)$ such that we have convergence in distribution:
    \[
        \lim_{\delta \to 0} \pi^\delta \to \pi^*
    \]
    where $\pi^\delta$ is defined with respect to the soft Q function induced by the prior $\pi$. Moreover, the policy $\pi^*$ can be explicitly characterised as the uniform distribution over $A^* := \{a^* \in \A: Q^\pi(s, a^*) = \max_{a \in \A} Q^\pi(s, a)\}$.
\end{proposition}

We can now tie together the notions of \emph{coherence} and \emph{optimality}.

\begin{corollary}
    If the prior $p$ is an optimal policy for the Markov chain, then limiting policy $\pi^*$ from \Cref{lemma:optimal-maxent} is also optimal, which is not necessarily true for non-optimal priors.
\end{corollary}

What about other policies? We can still use the limiting distance to the Boltzmann-rational policies to determine the optimality.
\begin{corollary}\label{prop:kappa-optimal}
    A necessary condition for policy $\pi(a|s)$ to be optimal is that:
    \[
    \lim_{\delta \to 0} \kappa_\delta(\pi) < \infty
    \]
\end{corollary}

We note that the above~\Cref{prop:kappa-optimal} does not extend to a sufficient condition. To see this, it is enough to consider a case of $T=2$ MDP, such as e.g.~\Cref{example:counter}. If the prior $\pi$ puts zero weight on the action $\nearrow$ in the state $\emptyset$. Then, even though in state $\mountain$, $Q^\pi$ is independent of $\pi$ and therefore being $Q^\pi_t$-greedy is equivalent to optimality, it does not extend to $t=0$.

We also note that the Boltzmann coherence is a combination of being soft-conditioned and coherent. In the case of MDPs with time horizon $T = 1$, the coherence requirement disappears, and the condition comes down to simply optimising a reinforcement learning problem with a KL penalty. We follow~\cite{korbak_rl_2022} in stating the following.
\begin{proposition}[RL with KL penalties]\label{prop:rlwithkl}
    Given an MDP with time horizon $T = 1$, a reward $r(s,a)$ and a prior $p(a|s)$, the Boltzmann-coherent policy $\pi(a|s)$ can be derived by maximising the KL-regularised return $J(\pi) = \Expect{}{r(s,a)} - \KL(\pi(a|s)||p(a|s))$. The resulting policy is of a form:
    \[
    \pi^*(a|s)\ \propto\ p(a|s)\exp(r(s,a))
    \]
\end{proposition}

\section{REMOVING INCOHERENCE}\label{sec:removing-incoherence}
Having established the notions of incoherence, we now turn to the topic of removing it through the process of fine-tuning policies on their own actions. We first establish the notion of goal-conditioning a prior over actions.

\begin{definition}[Goal conditioning]\label{def:goal-conditioning}
    Given transition dynamics $\tau(s_{t+1}|a_{t}, s_{t})$, a prior $p(a_t|s_t)$ and a non-positive reward $r(s, a)$, we say that a policy $\pi$ is \emph{given by goal conditioning}, if we have that:
    \[
    \pi(a_t|s_t) = p(a_t|s_t, \OO_{t:T} = 1)
    \]
    where the right-hand side is defined through the joint distribution of $(\tau, p)$ and the auxiliary \emph{optimality variable} $\OO_t$ has a Bernoulli distribution of:
    \[
    p(\OO_t = 1|s_t,a_t) = e^{r(s_t,a_t)}
    \]
\end{definition}

Simply conditioning on the optimality does not guarantee coherence.

\begin{proposition}
    A policy given by goal conditioning is not necessarily coherent for any choice of $\delta$.
\end{proposition}

Indeed, the counterexample is given by~\Cref{example:two-cards}. However, we might hope to fix this by iterating the goal conditioning, that is, retraining a policy on its own actions. To formalise this process, we give the following definition.

\begin{definition}[Control-as-inference]
    We define the control-as-inference operator $\GG$ on the space of probability distributions $p(a_t|s_t)$ as 
    \[
    \GG(p) = p(a_t|s_t, \OO_{t:T} = 1)
    \]
    where the left-hand side is given jointly by $p, \tau, \OO$. This gives a sequence of policies $\pi^\GG_i$ defined recursively as:
    \begin{align*}
        \pi^\GG_0(a_t|s_t) = p(a_t|s_t) \qquad \pi^\GG_{k + 1} = \GG(\pi^\GG_{k})
    \end{align*}
\end{definition}

We note that this process does not condition the dynamics $\tau$ of the MDP. In other words, it might be thought to be implemented as iteratively changing the joint by collecting rollouts from the conditioned policy, both those successful and unsuccessful, without rejection sampling (conditioned on optimality). Even still, we have the following property:

\begin{proposition}[Strong return improvement lemma]\label{prop:strong-improvement}
    The sequence of policies $(\pi^\GG_i)_{i = 0, 1, \dots}$ given by the control-by-inference improves its return monotonically, that is:
    \[
    J(\pi^\GG_{i + 1}) \geq J(\pi^\GG_{i}) 
    \]
\end{proposition}

The fact that the policies monotonically improve gives a reason to suspect that the limiting policy (which exists by a compactness argument) will be an optimal policy for the reward $r$. Indeed, we can show the following.

\begin{proposition}[{\citep[Theorem 3.6]{douglas2024limitations}}]\label{prop:limitations-thm}
    There exists some policy $\pi^\GG$ that the sequence $\pi^\GG_i$ converges to. That is, we have: 
    \[
        \lim_{i \to \infty} \KL(\pi^\GG_i||\pi^\GG) = 0
    \]
    Moreover, if the prior $p(a|s)$ has full support over all actions $a$ in all states $s$, then $\pi^\GG$ is optimal.
\end{proposition}

This property is a consequence of the equivalence we develop in~\Cref{sec:equivalence}; a direct proof can be found in~\citep[Appendix A]{douglas2024limitations}.

Another formulation of the sequence of increasingly improving policies is given by increasing the temperature of the energy distribution of the optimality variable $\OO$.

\begin{definition}[RL with temperature]\label{def:temperature}
    Given a prior $p(a_t|s_t)$, a reward $r(s_t,a_t)$ and an \emph{(inverse) temperature} function $\alpha: \mathbb{N} \to \mathbb{R}_+$, we define a sequence of policies:
    \[
    \pi_{\alpha(k)}(a_t|s_t) = p(a_t|s_t, \OO^{\alpha(k)}_{t:T} = 1) 
    \]
    where:
    \[
    p(\OO^{\alpha(k)}_t|s_t,a_t) = \exp\left({\alpha(k)\cdot{r(s_t,a_t)}}\right)
    \]
\end{definition}

We note that the process of increasing temperature in this way behaves differently from what taking $\delta \to 0$ in we explored in the last section's ~\Cref{def:boltzmann-incoherence}. There, we took a fixed $Q$ function defined by a particular policy $\pi$, and applied softmax. Here, the policy itself is changing, so we see progressively more precise $Q$ functions of a form $Q^{\pi_{\alpha(k)}}$, as (inverse) temperature $\alpha(k)$ increases.

Finally, we talk about folding the posterior into the reward. \citet{levine_reinforcement_2018} discusses a trick for folding \emph{the prior} into the reward. If soft $Q$ and $V$ functions are computed as in~\Cref{def:softvq}, one can disregard a prior over actions $p(a_t|s_t)$ by modifying the reward function $r(s_t, a_t)$ (specifying the Bernoulli distribution of $\mathcal{O}_t$) to be:
\[
\hat{r}(s_t, a_t) = r(s_t, a_t) + \log p(a_t|s_t)
\]
This defines an equivalent probabilistic model in which the prior $p(a_t|s_t)$ is uniform over actions, simplifying the calculations. Our point of interest lies in the fact that the coherent definition of the soft $V$ function uses \emph{the posterior} instead of prior:
\[
V^\pi(s_t) = \log \mathbb{E}_{a_t \sim \pi}[\exp Q^\pi(a_t, s_t)]
\]
as we discussed in the introduction, pointing to the difference between Questions (1) and (2).

Because the expectation is over the policy $\pi(a_t|s_t) = p(a_t|s_t, \mathcal{O}_{t:T})$ and \emph{not} the prior $p(a_t|s_t)$, it is no longer that simple to fold it into the reward. Since the reward (and therefore, the distributions of $\mathcal{O}_t$'s) changed, it has a downstream effect of changing the posterior. Thus, iterative process is needed.

\begin{definition}[Folded sequence]\label{def:folded-orig}
    For a prior $p(a_t|s_t)$ and a reward $r(s_t,a_t)$ we define a sequence of policies $\pi^\FF_k$ and rewards $r_k$ recursively, to be
    \[
        r_0 = r \qquad \pi^\FF_0 = p
    \]
    \[
     r_{k + 1}(s_t,a_t) = r_0(s_t,a_t) + \log p(a_t|s_t, \OO^{(k)}_{t:T} = 1)
    \]
    \[
     \pi^\FF_{k+1}(a_t|s_t) = p(a_t|s_t, \OO^{(k)}_{t:T} = 1)
    \]
    \[
    p(\OO^{(k)}_{t} = 1|a_t,s_t) = \exp\left(r_k(a_t, s_t)\right)
    \]
\end{definition}

It turns out that this iterative process is an equivalent way of formulating the re-training the model on its own output. We can also fold the reward recursively into the previous reward:
\begin{definition}[Folded sequence, cumulative]\label{def:folded}
    For a prior $p(a_t|s_t)$ and a reward $r(s_t,a_t)$ we define a sequence of policies $\pi^\HH_k$, the optimality variables $\mathcal{O}$ as in~\Cref{def:folded-orig}, with the only change being $r_{k+1}$:
    \[
     r_{k + 1}(s_t,a_t) = r_k(s_t,a_t) + \log p(a_t|s_t, \OO^{(k)}_{t:T} = 1)
    \]
\end{definition}
This version is, however, only equivalent to the other modes of retraining under the assumption of deterministic transition dynamics $\tau$, as we prove in the next~\Cref{sec:equivalence}.

\subsection{Equivalence and corollaries}\label{sec:equivalence}

After introducing the three ways of constructing a sequence of policies: fine-tuning policies on their own actions, folding the posterior into the reward, and conditioning on the reward with an increased temperature, we are now able to connect those dynamics. In case of deterministic dynamics, all three coincide, giving the same policy trajectories. In case of stochastic dynamics, it is only true for the first two.

\begin{theorem}[Optimisation pressure, thrice]\label{theorem:equivalence-thrice}
    For all priors $p(a|s)$, all non-positive reward functions $r(s_t,a_t)$ and deterministic transition dynamics $\tau(s_{t+1}|s_t, a_t)$, for all $k \in \mathbb{N}_+$ and for $\alpha(k) = 2^k$, we have:
    \[
    \pi_{\alpha(k)} = \pi^\GG_{2^k} = \pi^\FF_{2^k} = \pi^\HH_k
    \]
    Additionally, for stochastic transition dynamics we have:
    \[
    \pi^\FF_k = \pi^\mathcal{G}_{k}
    \]
\end{theorem}

\begin{example}\label{example:counter}
    Transition dynamics such that $\pi_{\alpha(k)} \neq \pi^\GG_k$ are simple to construct. For example, take a three-state, two action Markov decision process with initial state $\emptyset$ and a uniform prior policy, with transition dynamics depicted on the diagram below:
\[\begin{tikzcd}
    {s_1} & \emptyset & {s_2}
    \arrow["{\frac{3}{4}}"', color={rgb,255:red,0;green,108;blue,238}, curve={height=12pt}, dashed, from=1-2, to=1-1]
    \arrow["{\frac{1}{2}}", curve={height=-12pt}, from=1-2, to=1-1]
    \arrow["{\frac{1}{4}}", color={rgb,255:red,0;green,108;blue,238}, curve={height=-12pt}, dashed, from=1-2, to=1-3]
    \arrow["{\frac{1}{2}}"', curve={height=12pt}, from=1-2, to=1-3]
\end{tikzcd}\]
where red dotted lines correspond to action $a_1$, and solid black lines to action $a_2$. We also assume that:
\[
r(s_1) = \log \frac{1}{3} \qquad r(s_2) = \log \frac{2}{3}
\]
Computing $\pi^\mathcal{G}_1(\cdot|\emptyset)$ and $\pi_{\alpha(2)}(\cdot|\emptyset)$ we get:
\[
\frac{25}{61} = \pi^\mathcal{G}_1(a_1|\emptyset) \neq \pi_{\alpha(2)}(a_1|\emptyset) =\frac{7}{17} 
\]
\end{example}
From this equivalence, and additional properties of each of the ways of constructing the sequence, we can easily derive results that would require laborious proofs otherwise. For example, we might easily show that the policies converge to an optimal policy, and explicitly calculate the rate of convergence, using the causal entropy regularisation characterisation of $\pi_\alpha(k)$.

\begin{corollary}[Return improvement rate]\label{corr:rate-of-convergence}
    In case of deterministic dynamics $\tau$, given a sequence of policies $\pi^\GG_i$, for $k \in \mathbb{N}_+$ have that:
    \[
    J(\pi^\GG_{k}) - J(\pi^\GG_{k-1}) = \frac{1}{{k}} \cdot \frac{J'(\pi^\GG_{k}) \cdot \hat{\HH}'(\pi^\GG_{k})}{J''(\pi^\GG_k) + \frac{1}{{k}} \hat{\HH}''(\pi^\GG_{k})} + O\left(\frac{1}{{k}^2}\right)
    \]
    where $\hat{\HH}(\pi)$ denotes the causal entropy of policy $\pi$, and the derivatives are taken with respect to the temperature $\alpha$.
\end{corollary}

Incoherence disappears at the limit of retraining.

\begin{corollary}\label{corr:limit}
    We have $\lim_{(\delta, i) \to (0,\infty)} \kappa^\delta(\pi^\GG_i) = 0$.
\end{corollary}

The significance of the correspondence should be possible to also be understood through the lenses of training-inference trade-off. Decreasing the temperature in the information-bounded RL can be implemented as top-$k$ rejection-sampling based methods (such as Speculative Rejection~\citep{sun_fast_2024}). On the other hand, we might want to pay the cost of inference once, during training, by retraining the model on its own outputs. Prior work has empirically verified that aligning language models with RLHF produces a similar behavior to best-of-$k$ rejection sampling with respect to the reward model~\citep{bai_training_2022, stiennon_learning_2022}. Our results confirm those results, and allow for precise trade-off estimation in the multi-step environment, with no difference on the margin with each phase of re-training equivalent to linearly increasing $k$.

\section{INCOHERENCE AND EFFECTIVE HORIZON}\label{sec:effective}

\begin{figure*}
\begin{subfigure}{0.5\textwidth}
    \includegraphics[width=1\linewidth]{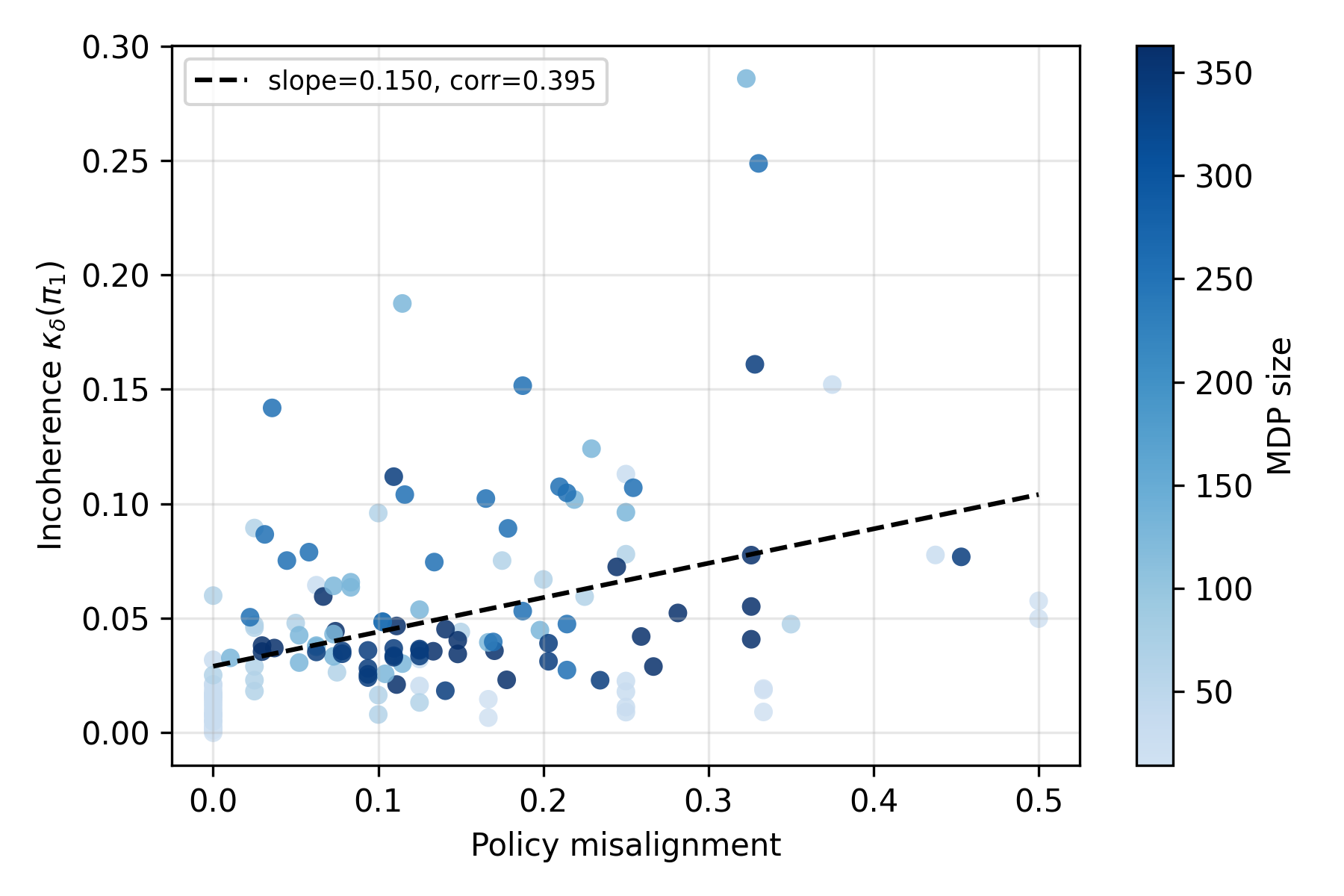}
\caption{\phantom{a}}
\label{fig:misalignment-incoherence-temp1}
\end{subfigure}
\begin{subfigure}{0.5\textwidth}
    \includegraphics[width=1\linewidth]{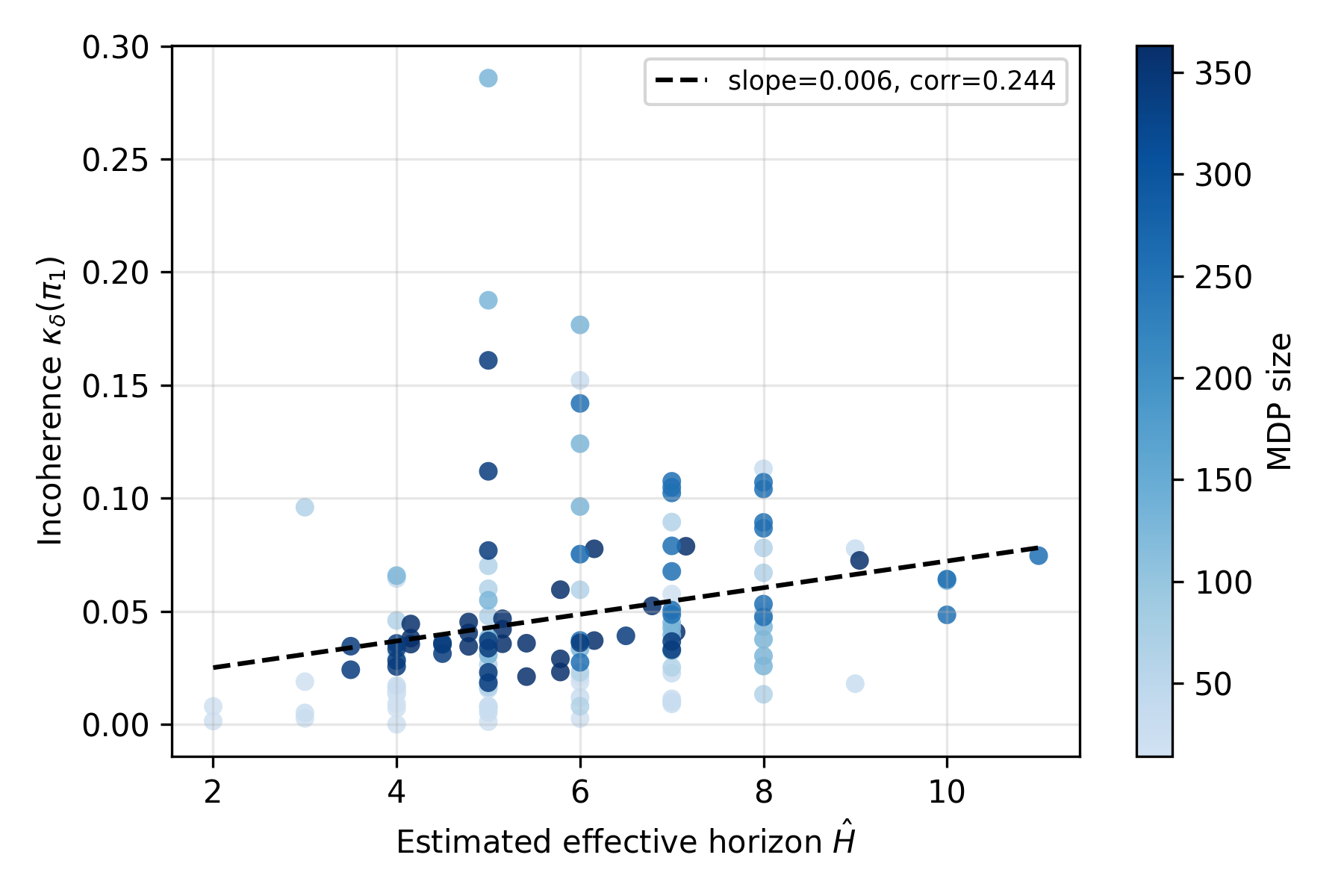}
\caption{\phantom{a}}
\label{fig:eh-incoherence-pi1}
\end{subfigure}

\caption{Empirically observed significant positive correlations between incoherence and (a) misalignment score, and (b) estimated effective horizon, in a dataset of 140 randomly chosen small MDPs, between $4$ and $122$ states.}
\label{fig:image2}
\end{figure*}

We have started the discussion in~\Cref{sec:introduction} by pointing out the difference between Questions (1) and (2), with~\Cref{example:two-cards} showing that the autoregressive  policy derived by conditioning on the reward answers (1). Concurrently with the development of this work, \citet{laidlaw_theoretical_2023,laidlaw_bridging_2024} performed extensive experiments to find a reason for which deep RL works in some environments, but not others. They empirically arrived, and later formalised, the notion \emph{effective horizon} (EH): a quantitative measure of the hardness of a given environment. We quote:

{\leftskip=0.5cm \rightskip=0.5cm
\emph{When actions with the highest Q-values under the random policy also have the highest Q-values under the optimal policy (i.e. when it is optimal to be greedy on the random policy's Q function), deep RL tends to succeed; when they don't, deep RL tends to fail.}
\par}

We note the striking similarity to the setup studied here: we rephrase their finding to state that \emph{deep RL tends to succeed when there is no difference between answers to Questions (1) and (2); it tends to fail otherwise.} This suggests that deep RL algorithms approximate naive autoregressive control-as-inference policies, which provides a theoretical justification of the effective horizon result. EH and incoherence are nevertheless different on a technical level, making a direct theoretical comparison difficult. Instead, we opt to compare them experimentally.

\paragraph{Misalignment and incoherence} In the first approach, we compare incoherence to the \emph{policy misalignment}, which is one operationalisation of the difference between Questions \emph{(1)} and \emph{(2)} aligned with the EH perspective above. Given an MDP, we define the misalignment $m \in [0, 1]$ to be the proportion of actions being different between $\pi_{\text{rand}}$ and $\pi^*$, where $\pi_{\text{rand}}$ selects the greedy action under the random policy’s Q‑function $Q^{\pi_{\mathrm{rand}}}$, averaged over the random policy occupancy measure, motivated by~\cite{laidlaw_bridging_2024}, and $\pi^*$ is a deterministic optimal policy (in our experiments, both are a.s. unique). For a set of 140 randomly chosen small MDP, we computed $m$ and Boltzmann incoherence $\kappa_\delta(\pi^G)$ for the goal‑conditioned uniform prior $\pi^\mathcal{G} = \mathcal{G}(\pi_\text{rand})$, for temperature $\delta = 1$ (for a discussion on other temperatures, please see~\Cref{app:experiments}). We present this in~\Cref{fig:misalignment-incoherence-temp1}. The correlation between $m$ and $\kappa$ is then $0.395$ (95\% CI $[0.245, 0.526]$), and isn't sensitive to outliers. In other words, environments where being greedy on $Q^{\pi_{\mathrm{rand}}}$ is suboptimal (large misalignment) also tend to be those where naive goal‑conditioning produces larger incoherence. This is qualitatively aligned with the hypothesis put forward above.

\paragraph{Effective horizon vs incoherence}

In the second approach, for a more direct comparison, we implemented a simple Monte Carlo-based effective horizon estimator $\hat{H}$ for small deterministic-transition MDPs (for details on $\hat{H}$, please see~\Cref{app:eff-horizon}). We reused the same dataset of 140 small random MDPs from the first experiment, and we measured $\kappa_\delta(\pi)$ for the naively goal‑conditioned policy $\pi^\mathcal{G} = \mathcal{G}(\pi^\text{unif})$ on a uniform prior. We obtained the correlation between $\hat{H}$ and $\kappa_\delta(\pi^\mathcal{G}) = 0.244$ (95\% CI $[0.082, 0.394]$), plotted on~\Cref{fig:eh-incoherence-pi1}. We again checked that the result is not sensitive to outliers. In other words, environments with large effective horizon (where being greedy on $Q^{\pi_{\mathrm{rand}}}$ is misleading) often coincide with those where the conditioned policy is highly incoherent, which is again consistent with the picture put forward at the start of this section.

\section{DISCUSSION}\label{sec:discussion}

\paragraph{Incoherence as a diagnostic metric}
Incoherence, as a metric of the mismatch between the policy's own Q function and its behavior, 
could be used as an additional diagnostic for monitoring RL training progress. To test this, we ran an experiment to measure the correlation between return and incoherence. We note that a priori, a policy can be coherent but not optimal and vice versa, but we expect the increase of coherence to lead to better performance on average. Reusing the same dataset of 140 MDPs, starting from uniformly random policy, we applied the operator $\GG$ 4 times. Each iteration, we computed the expected return $J$ and incoherence $\kappa_\delta$, setting $\delta=1$. The average correlation between $J$ and negative-$\kappa$ for each experiment (then averaged across experiments) was $0.978$. On all trajectories flattened, the average correlation remained high at $0.863$, providing some evidence for its validity.
\paragraph{Limitations and future work}
We have only conducted experiments in small tabular MDP environments, and relied on the already established experimental results from the prior work. Synthetic data experiments with Decision Transformers or other sequence models are a natural next step; by training a model on a game such as chess and controlling the training dataset, it should be possible to verify to what extent do models such as DTs suffer from incoherence, as answers to Questions \emph{(1)} and \emph{(2)} can be computed exactly. On the theory side, extending our treatment to the case of infinite time horizon requires handling discount factor which is known to be more involved~\citep{haarnoja_reinforcement_2017, levine_reinforcement_2018}. We did not propose any methods of regularising models towards coherence during training, without the expensive re-training procedure, for example by including an auxiliary term in the loss, which would be an interesting direction for future work. Finally, our understanding of stochastic environments is currently limited to negative results, but it seems that in practice, RL still behaves well (both from our limited experiments and in~\citet{laidlaw_theoretical_2023}), which points towards understanding this as an important theoretical direction.

\subsubsection*{Acknowledgements}
The authors would like to thank Victoria Krakovna, Cassidy Laidlaw, Joar Skalse, Sam Staton, participants of the SPAR program, and anonymous reviewers for their feedback and comments on various drafts of this paper. The first author was supported by a grant from Open Philanthropy. Some of this work was performed as part of the SERI MATS program.

\bibliography{references}

@misc{shao_deepseekmath_2024,
	title = {{DeepSeekMath}: {Pushing} the {Limits} of {Mathematical} {Reasoning} in {Open} {Language} {Models}},
	shorttitle = {{DeepSeekMath}},
	url = {http://arxiv.org/abs/2402.03300},
	doi = {10.48550/arXiv.2402.03300},
	urldate = {2025-10-03},
	publisher = {arXiv},
	author = {Shao, Zhihong and Wang, Peiyi and Zhu, Qihao and Xu, Runxin and Song, Junxiao and Bi, Xiao and Zhang, Haowei and Zhang, Mingchuan and Li, Y. K. and Wu, Y. and Guo, Daya},
	month = apr,
	year = {2024},
	note = {arXiv:2402.03300 [cs]},
}

@article{radford_language_2019,
	title = {Language {Models} are {Unsupervised} {Multitask} {Learners}},
	language = {en},
	journal = {OpenAI blog},
	author = {Radford, Alec and Wu, Jeffrey and Child, Rewon and Luan, David and Amodei, Dario and Sutskever, Ilya},
	year = {2019},
}

@misc{openai_dota_2019,
	title = {Dota 2 with {Large} {Scale} {Deep} {Reinforcement} {Learning}},
	url = {http://arxiv.org/abs/1912.06680},
	doi = {10.48550/arXiv.1912.06680},
	urldate = {2024-01-22},
	publisher = {arXiv},
	author = {OpenAI and Berner, Christopher and Brockman, Greg and Chan, Brooke and Cheung, Vicki and Debiak, Przemysław and Dennison, Christy and Farhi, David and Fischer, Quirin and Hashme, Shariq and Hesse, Chris and Józefowicz, Rafal and Gray, Scott and Olsson, Catherine and Pachocki, Jakub and Petrov, Michael and Pinto, Henrique P. d O. and Raiman, Jonathan and Salimans, Tim and Schlatter, Jeremy and Schneider, Jonas and Sidor, Szymon and Sutskever, Ilya and Tang, Jie and Wolski, Filip and Zhang, Susan},
	month = dec,
	year = {2019},
	note = {arXiv:1912.06680 [cs, stat]},
}

@misc{sun_fast_2024,
	title = {Fast {Best}-of-{N} {Decoding} via {Speculative} {Rejection}},
	url = {http://arxiv.org/abs/2410.20290},
	doi = {10.48550/arXiv.2410.20290},
	urldate = {2025-10-02},
	publisher = {arXiv},
	author = {Sun, Hanshi and Haider, Momin and Zhang, Ruiqi and Yang, Huitao and Qiu, Jiahao and Yin, Ming and Wang, Mengdi and Bartlett, Peter and Zanette, Andrea},
	month = oct,
	year = {2024},
	note = {arXiv:2410.20290 [cs]},
}

@misc{srivastava_training_2021,
	title = {Training {Agents} using {Upside}-{Down} {Reinforcement} {Learning}},
	url = {http://arxiv.org/abs/1912.02877},
	doi = {10.48550/arXiv.1912.02877},
	urldate = {2025-10-02},
	publisher = {arXiv},
	author = {Srivastava, Rupesh Kumar and Shyam, Pranav and Mutz, Filipe and Jaśkowski, Wojciech and Schmidhuber, Jürgen},
	month = sep,
	year = {2021},
	note = {arXiv:1912.02877 [cs]},
}

@misc{strupl_upside-down_2022,
	title = {Upside-{Down} {Reinforcement} {Learning} {Can} {Diverge} in {Stochastic} {Environments} {With} {Episodic} {Resets}},
	url = {http://arxiv.org/abs/2205.06595},
	doi = {10.48550/arXiv.2205.06595},
	urldate = {2025-10-02},
	publisher = {arXiv},
	author = {Štrupl, Miroslav and Faccio, Francesco and Ashley, Dylan R. and Schmidhuber, Jürgen and Srivastava, Rupesh Kumar},
	month = may,
	year = {2022},
	note = {arXiv:2205.06595 [stat]},
}

@misc{yang_large_2024,
	title = {Large {Language} {Models} as {Optimizers}},
	url = {http://arxiv.org/abs/2309.03409},
	doi = {10.48550/arXiv.2309.03409},
	urldate = {2025-10-02},
	publisher = {arXiv},
	author = {Yang, Chengrun and Wang, Xuezhi and Lu, Yifeng and Liu, Hanxiao and Le, Quoc V. and Zhou, Denny and Chen, Xinyun},
	month = apr,
	year = {2024},
	note = {arXiv:2309.03409 [cs]},
}

@misc{paster_you_2022,
	title = {You {Can}'t {Count} on {Luck}: {Why} {Decision} {Transformers} and {RvS} {Fail} in {Stochastic} {Environments}},
	shorttitle = {You {Can}'t {Count} on {Luck}},
	url = {http://arxiv.org/abs/2205.15967},
	doi = {10.48550/arXiv.2205.15967},
	urldate = {2025-10-02},
	publisher = {arXiv},
	author = {Paster, Keiran and McIlraith, Sheila and Ba, Jimmy},
	month = nov,
	year = {2022},
	note = {arXiv:2205.15967 [cs]},
}

@inproceedings{peters_relative_2010,
	address = {Atlanta, Georgia},
	series = {{AAAI}'10},
	title = {Relative entropy policy search},
	urldate = {2025-10-02},
	booktitle = {Proceedings of the {Twenty}-{Fourth} {AAAI} {Conference} on {Artificial} {Intelligence}},
	publisher = {AAAI Press},
	author = {Peters, Jan and Mülling, Katharina and Altün, Yasemin},
	month = jul,
	year = {2010},
	pages = {1607--1612},
}

@misc{mccoy_embers_2023,
	title = {Embers of {Autoregression}: {Understanding} {Large} {Language} {Models} {Through} the {Problem} {They} are {Trained} to {Solve}},
	shorttitle = {Embers of {Autoregression}},
	url = {http://arxiv.org/abs/2309.13638},
	doi = {10.48550/arXiv.2309.13638},
	urldate = {2025-10-02},
	publisher = {arXiv},
	author = {McCoy, R. Thomas and Yao, Shunyu and Friedman, Dan and Hardy, Matthew and Griffiths, Thomas L.},
	month = sep,
	year = {2023},
	note = {arXiv:2309.13638 [cs]},
}

@misc{deng_residual_2020,
	title = {Residual {Energy}-{Based} {Models} for {Text} {Generation}},
	url = {http://arxiv.org/abs/2004.11714},
	doi = {10.48550/arXiv.2004.11714},
	urldate = {2025-10-02},
	publisher = {arXiv},
	author = {Deng, Yuntian and Bakhtin, Anton and Ott, Myle and Szlam, Arthur and Ranzato, Marc'Aurelio},
	month = apr,
	year = {2020},
	note = {arXiv:2004.11714 [cs]},
}

@misc{hao_training_2024,
	title = {Training {Large} {Language} {Models} to {Reason} in a {Continuous} {Latent} {Space}},
	url = {http://arxiv.org/abs/2412.06769},
	doi = {10.48550/arXiv.2412.06769},
	urldate = {2025-10-02},
	publisher = {arXiv},
	author = {Hao, Shibo and Sukhbaatar, Sainbayar and Su, DiJia and Li, Xian and Hu, Zhiting and Weston, Jason and Tian, Yuandong},
	month = dec,
	year = {2024},
	note = {arXiv:2412.06769 [cs]},
}

@misc{abdolmaleki_maximum_2018,
	title = {Maximum a {Posteriori} {Policy} {Optimisation}},
	url = {http://arxiv.org/abs/1806.06920},
	doi = {10.48550/arXiv.1806.06920},
	urldate = {2025-10-02},
	publisher = {arXiv},
	author = {Abdolmaleki, Abbas and Springenberg, Jost Tobias and Tassa, Yuval and Munos, Remi and Heess, Nicolas and Riedmiller, Martin},
	month = jun,
	year = {2018},
	note = {arXiv:1806.06920 [cs]},
}

@misc{belousov_kl_2017,
	title = {{KL} between trajectory distributions vs {KL} between policies · {Boris} {Belousov}},
	url = {http://www.boris-belousov.net/2017/04/16/KL-between-trajectories-and-policies/},
	urldate = {2025-09-06},
	author = {Belousov, Boris},
	year = {2017},
}

@inproceedings{douglas2024limitations,
	title = {Limitations of agents simulated by predictive models},
	url = {https://openreview.net/forum?id=4gcoAjKaLf},
	booktitle = {{ICLR} 2024 workshop on large language model ({LLM}) agents},
	author = {Douglas, Raymond and Karwowski, Jacek and Bae, Chan and Draguns, Andis and Krakovna, Victoria},
	year = {2024},
}

@misc{yang_auto-gpt_2023,
	title = {Auto-{GPT} for {Online} {Decision} {Making}: {Benchmarks} and {Additional} {Opinions}},
	shorttitle = {Auto-{GPT} for {Online} {Decision} {Making}},
	url = {http://arxiv.org/abs/2306.02224},
	doi = {10.48550/arXiv.2306.02224},
	urldate = {2024-05-22},
	publisher = {arXiv},
	author = {Yang, Hui and Yue, Sifu and He, Yunzhong},
	month = jun,
	year = {2023},
	note = {arXiv:2306.02224 [cs]},
}

@misc{significant_gravitas_autogpt_2024,
	title = {{AutoGPT}},
	copyright = {MIT},
	url = {https://github.com/Significant-Gravitas/AutoGPT},
	urldate = {2024-05-22},
	author = {{Significant Gravitas}},
	month = may,
	year = {2024},
	note = {original-date: 2023-03-16T09:21:07Z},
}

@article{browne_survey_2012,
	title = {A {Survey} of {Monte} {Carlo} {Tree} {Search} {Methods}},
	volume = {4},
	issn = {1943-0698},
	url = {https://ieeexplore.ieee.org/document/6145622},
	doi = {10.1109/TCIAIG.2012.2186810},
	number = {1},
	urldate = {2024-05-22},
	journal = {IEEE Transactions on Computational Intelligence and AI in Games},
	author = {Browne, Cameron B. and Powley, Edward and Whitehouse, Daniel and Lucas, Simon M. and Cowling, Peter I. and Rohlfshagen, Philipp and Tavener, Stephen and Perez, Diego and Samothrakis, Spyridon and Colton, Simon},
	month = mar,
	year = {2012},
	note = {Conference Name: IEEE Transactions on Computational Intelligence and AI in Games},
	pages = {1--43},
}

@inproceedings{ziebart_maximum_2008,
	address = {Chicago, Illinois},
	series = {{AAAI}'08},
	title = {Maximum entropy inverse reinforcement learning},
	isbn = {978-1-57735-368-3},
	urldate = {2024-05-22},
	booktitle = {Proceedings of the 23rd national conference on {Artificial} intelligence - {Volume} 3},
	publisher = {AAAI Press},
	author = {Ziebart, Brian D. and Maas, Andrew and Bagnell, J. Andrew and Dey, Anind K.},
	month = jul,
	year = {2008},
	pages = {1433--1438},
}

@misc{anthony_thinking_2017,
	title = {Thinking {Fast} and {Slow} with {Deep} {Learning} and {Tree} {Search}},
	url = {http://arxiv.org/abs/1705.08439},
	doi = {10.48550/arXiv.1705.08439},
	urldate = {2024-05-22},
	publisher = {arXiv},
	author = {Anthony, Thomas and Tian, Zheng and Barber, David},
	month = dec,
	year = {2017},
	note = {arXiv:1705.08439 [cs]},
}

@misc{brandfonbrener_when_2023,
	title = {When does return-conditioned supervised learning work for offline reinforcement learning?},
	url = {http://arxiv.org/abs/2206.01079},
	doi = {10.48550/arXiv.2206.01079},
	urldate = {2024-05-22},
	publisher = {arXiv},
	author = {Brandfonbrener, David and Bietti, Alberto and Buckman, Jacob and Laroche, Romain and Bruna, Joan},
	month = jan,
	year = {2023},
	note = {arXiv:2206.01079 [cs]},
}

@article{schrittwieser_mastering_2020,
	title = {Mastering {Atari}, {Go}, {Chess} and {Shogi} by {Planning} with a {Learned} {Model}},
	volume = {588},
	issn = {0028-0836, 1476-4687},
	url = {http://arxiv.org/abs/1911.08265},
	doi = {10.1038/s41586-020-03051-4},
	number = {7839},
	urldate = {2024-05-22},
	journal = {Nature},
	author = {Schrittwieser, Julian and Antonoglou, Ioannis and Hubert, Thomas and Simonyan, Karen and Sifre, Laurent and Schmitt, Simon and Guez, Arthur and Lockhart, Edward and Hassabis, Demis and Graepel, Thore and Lillicrap, Timothy and Silver, David},
	month = dec,
	year = {2020},
	note = {arXiv:1911.08265 [cs, stat]},
	pages = {604--609},
}

@inproceedings{laidlaw_theoretical_2023,
	title = {A {Theoretical} {Explanation} of {Deep} {RL} {Performance} in {Stochastic} {Environments}},
	url = {https://openreview.net/forum?id=5ES5Hdlbxw},
	urldate = {2024-05-21},
	booktitle = {The {Twelfth} {International} {Conference} on {Learning} {Representations}},
	author = {Laidlaw, Cassidy and Zhu, Banghua and Russell, Stuart and Dragan, Anca},
	year = {2023},
}

@misc{gleave_primer_2022,
	title = {A {Primer} on {Maximum} {Causal} {Entropy} {Inverse} {Reinforcement} {Learning}},
	url = {http://arxiv.org/abs/2203.11409},
	urldate = {2024-02-20},
	publisher = {arXiv},
	author = {Gleave, Adam and Toyer, Sam},
	month = mar,
	year = {2022},
	note = {arXiv:2203.11409 [cs]},
}

@article{silver_mastering_2017,
	title = {Mastering the game of {Go} without human knowledge},
	volume = {550},
	copyright = {2017 Macmillan Publishers Limited, part of Springer Nature. All rights reserved.},
	issn = {1476-4687},
	url = {https://www.nature.com/articles/nature24270},
	doi = {10.1038/nature24270},
	language = {en},
	number = {7676},
	urldate = {2024-02-02},
	journal = {Nature},
	author = {Silver, David and Schrittwieser, Julian and Simonyan, Karen and Antonoglou, Ioannis and Huang, Aja and Guez, Arthur and Hubert, Thomas and Baker, Lucas and Lai, Matthew and Bolton, Adrian and Chen, Yutian and Lillicrap, Timothy and Hui, Fan and Sifre, Laurent and van den Driessche, George and Graepel, Thore and Hassabis, Demis},
	month = oct,
	year = {2017},
	note = {Number: 7676
Publisher: Nature Publishing Group},
	pages = {354--359},
}

@misc{laidlaw_bridging_2024,
	title = {Bridging {RL} {Theory} and {Practice} with the {Effective} {Horizon}},
	url = {http://arxiv.org/abs/2304.09853},
	urldate = {2024-01-29},
	publisher = {arXiv},
	author = {Laidlaw, Cassidy and Russell, Stuart and Dragan, Anca},
	month = jan,
	year = {2024},
	note = {arXiv:2304.09853 [cs, stat]},
}

@article{vinyals_grandmaster_2019,
	title = {Grandmaster level in {StarCraft} {II} using multi-agent reinforcement learning},
	volume = {575},
	copyright = {2019 The Author(s), under exclusive licence to Springer Nature Limited},
	issn = {1476-4687},
	url = {https://www.nature.com/articles/s41586-019-1724-z},
	doi = {10.1038/s41586-019-1724-z},
	language = {en},
	number = {7782},
	urldate = {2024-01-22},
	journal = {Nature},
	author = {Vinyals, Oriol and Babuschkin, Igor and Czarnecki, Wojciech M. and Mathieu, Michaël and Dudzik, Andrew and Chung, Junyoung and Choi, David H. and Powell, Richard and Ewalds, Timo and Georgiev, Petko and Oh, Junhyuk and Horgan, Dan and Kroiss, Manuel and Danihelka, Ivo and Huang, Aja and Sifre, Laurent and Cai, Trevor and Agapiou, John P. and Jaderberg, Max and Vezhnevets, Alexander S. and Leblond, Rémi and Pohlen, Tobias and Dalibard, Valentin and Budden, David and Sulsky, Yury and Molloy, James and Paine, Tom L. and Gulcehre, Caglar and Wang, Ziyu and Pfaff, Tobias and Wu, Yuhuai and Ring, Roman and Yogatama, Dani and Wünsch, Dario and McKinney, Katrina and Smith, Oliver and Schaul, Tom and Lillicrap, Timothy and Kavukcuoglu, Koray and Hassabis, Demis and Apps, Chris and Silver, David},
	month = nov,
	year = {2019},
	note = {Number: 7782
Publisher: Nature Publishing Group},
	pages = {350--354},
}

@misc{touvron_llama_2023,
	title = {{LLaMA}: {Open} and {Efficient} {Foundation} {Language} {Models}},
	shorttitle = {{LLaMA}},
	url = {http://arxiv.org/abs/2302.13971},
	urldate = {2024-01-22},
	publisher = {arXiv},
	author = {Touvron, Hugo and Lavril, Thibaut and Izacard, Gautier and Martinet, Xavier and Lachaux, Marie-Anne and Lacroix, Timothée and Rozière, Baptiste and Goyal, Naman and Hambro, Eric and Azhar, Faisal and Rodriguez, Aurelien and Joulin, Armand and Grave, Edouard and Lample, Guillaume},
	month = feb,
	year = {2023},
	note = {arXiv:2302.13971 [cs]},
}

@misc{haarnoja_soft_2018,
	title = {Soft {Actor}-{Critic}: {Off}-{Policy} {Maximum} {Entropy} {Deep} {Reinforcement} {Learning} with a {Stochastic} {Actor}},
	shorttitle = {Soft {Actor}-{Critic}},
	url = {http://arxiv.org/abs/1801.01290},
	urldate = {2024-01-22},
	publisher = {arXiv},
	author = {Haarnoja, Tuomas and Zhou, Aurick and Abbeel, Pieter and Levine, Sergey},
	month = aug,
	year = {2018},
	note = {arXiv:1801.01290 [cs, stat]},
}

@inproceedings{macleod_game_2005,
	address = {New York, NY, USA},
	series = {{ACE} '05},
	title = {Game design through self-play experiments},
	isbn = {978-1-59593-110-8},
	url = {https://doi.org/10.1145/1178477.1178572},
	doi = {10.1145/1178477.1178572},
	urldate = {2024-01-03},
	booktitle = {Proceedings of the 2005 {ACM} {SIGCHI} {International} {Conference} on {Advances} in computer entertainment technology},
	publisher = {Association for Computing Machinery},
	author = {Macleod, Alasdair},
	month = jun,
	year = {2005},
	pages = {421--428},
}

@article{silver_general_2018,
	title = {A general reinforcement learning algorithm that masters chess, shogi, and {Go} through self-play},
	volume = {362},
	url = {https://www.science.org/doi/full/10.1126/science.aar6404},
	doi = {10.1126/science.aar6404},
	number = {6419},
	urldate = {2024-01-04},
	journal = {Science},
	author = {Silver, David and Hubert, Thomas and Schrittwieser, Julian and Antonoglou, Ioannis and Lai, Matthew and Guez, Arthur and Lanctot, Marc and Sifre, Laurent and Kumaran, Dharshan and Graepel, Thore and Lillicrap, Timothy and Simonyan, Karen and Hassabis, Demis},
	month = dec,
	year = {2018},
	note = {Publisher: American Association for the Advancement of Science},
	pages = {1140--1144},
}

@misc{schulman_proximal_2017,
	title = {Proximal {Policy} {Optimization} {Algorithms}},
	url = {http://arxiv.org/abs/1707.06347},
	doi = {10.48550/arXiv.1707.06347},
	urldate = {2024-01-02},
	publisher = {arXiv},
	author = {Schulman, John and Wolski, Filip and Dhariwal, Prafulla and Radford, Alec and Klimov, Oleg},
	month = aug,
	year = {2017},
	note = {arXiv:1707.06347 [cs]},
}

@misc{bai_constitutional_2022,
	title = {Constitutional {AI}: {Harmlessness} from {AI} {Feedback}},
	shorttitle = {Constitutional {AI}},
	url = {http://arxiv.org/abs/2212.08073},
	doi = {10.48550/arXiv.2212.08073},
	urldate = {2024-01-02},
	publisher = {arXiv},
	author = {Bai, Yuntao and Kadavath, Saurav and Kundu, Sandipan and Askell, Amanda and Kernion, Jackson and Jones, Andy and Chen, Anna and Goldie, Anna and Mirhoseini, Azalia and McKinnon, Cameron and Chen, Carol and Olsson, Catherine and Olah, Christopher and Hernandez, Danny and Drain, Dawn and Ganguli, Deep and Li, Dustin and Tran-Johnson, Eli and Perez, Ethan and Kerr, Jamie and Mueller, Jared and Ladish, Jeffrey and Landau, Joshua and Ndousse, Kamal and Lukosuite, Kamile and Lovitt, Liane and Sellitto, Michael and Elhage, Nelson and Schiefer, Nicholas and Mercado, Noemi and DasSarma, Nova and Lasenby, Robert and Larson, Robin and Ringer, Sam and Johnston, Scott and Kravec, Shauna and Showk, Sheer El and Fort, Stanislav and Lanham, Tamera and Telleen-Lawton, Timothy and Conerly, Tom and Henighan, Tom and Hume, Tristan and Bowman, Samuel R. and Hatfield-Dodds, Zac and Mann, Ben and Amodei, Dario and Joseph, Nicholas and McCandlish, Sam and Brown, Tom and Kaplan, Jared},
	month = dec,
	year = {2022},
	note = {arXiv:2212.08073 [cs]},
}

@misc{andreas_language_2022,
	title = {Language {Models} as {Agent} {Models}},
	url = {http://arxiv.org/abs/2212.01681},
	urldate = {2024-01-02},
	publisher = {arXiv},
	author = {Andreas, Jacob},
	month = dec,
	year = {2022},
	note = {arXiv:2212.01681 [cs]},
}

@misc{odonoghue_making_2020,
	title = {Making {Sense} of {Reinforcement} {Learning} and {Probabilistic} {Inference}},
	url = {http://arxiv.org/abs/2001.00805},
	urldate = {2024-01-01},
	publisher = {arXiv},
	author = {O'Donoghue, Brendan and Osband, Ian and Ionescu, Catalin},
	month = nov,
	year = {2020},
	note = {arXiv:2001.00805 [cs]},
}

@misc{jeon_reward-rational_2020,
	title = {Reward-rational (implicit) choice: {A} unifying formalism for reward learning},
	shorttitle = {Reward-rational (implicit) choice},
	url = {http://arxiv.org/abs/2002.04833},
	urldate = {2024-01-01},
	publisher = {arXiv},
	author = {Jeon, Hong Jun and Milli, Smitha and Dragan, Anca D.},
	month = dec,
	year = {2020},
	note = {arXiv:2002.04833 [cs]},
}

@inproceedings{todorov_linearly-solvable_2006,
	title = {Linearly-solvable {Markov} decision problems},
	volume = {19},
	url = {https://papers.nips.cc/paper_files/paper/2006/hash/d806ca13ca3449af72a1ea5aedbed26a-Abstract.html},
	urldate = {2024-01-01},
	booktitle = {Advances in {Neural} {Information} {Processing} {Systems}},
	publisher = {MIT Press},
	author = {Todorov, Emanuel},
	year = {2006},
}

@misc{stiennon_learning_2022,
	title = {Learning to summarize from human feedback},
	url = {http://arxiv.org/abs/2009.01325},
	doi = {10.48550/arXiv.2009.01325},
	urldate = {2023-12-31},
	publisher = {arXiv},
	author = {Stiennon, Nisan and Ouyang, Long and Wu, Jeff and Ziegler, Daniel M. and Lowe, Ryan and Voss, Chelsea and Radford, Alec and Amodei, Dario and Christiano, Paul},
	month = feb,
	year = {2022},
	note = {arXiv:2009.01325 [cs]},
}

@misc{hubinger_conditioning_2023,
	title = {Conditioning {Predictive} {Models}: {Risks} and {Strategies}},
	shorttitle = {Conditioning {Predictive} {Models}},
	url = {http://arxiv.org/abs/2302.00805},
	doi = {10.48550/arXiv.2302.00805},
	urldate = {2023-12-31},
	publisher = {arXiv},
	author = {Hubinger, Evan and Jermyn, Adam and Treutlein, Johannes and Hudson, Rubi and Woolverton, Kate},
	month = feb,
	year = {2023},
	note = {arXiv:2302.00805 [cs]},
}

@misc{shanahan_role-play_2023,
	title = {Role-{Play} with {Large} {Language} {Models}},
	url = {https://arxiv.org/abs/2305.16367v1},
	language = {en},
	urldate = {2023-12-31},
	journal = {arXiv.org},
	author = {Shanahan, Murray and McDonell, Kyle and Reynolds, Laria},
	month = may,
	year = {2023},
}

@misc{bai_training_2022,
	title = {Training a {Helpful} and {Harmless} {Assistant} with {Reinforcement} {Learning} from {Human} {Feedback}},
	url = {http://arxiv.org/abs/2204.05862},
	doi = {10.48550/arXiv.2204.05862},
	urldate = {2023-12-31},
	publisher = {arXiv},
	author = {Bai, Yuntao and Jones, Andy and Ndousse, Kamal and Askell, Amanda and Chen, Anna and DasSarma, Nova and Drain, Dawn and Fort, Stanislav and Ganguli, Deep and Henighan, Tom and Joseph, Nicholas and Kadavath, Saurav and Kernion, Jackson and Conerly, Tom and El-Showk, Sheer and Elhage, Nelson and Hatfield-Dodds, Zac and Hernandez, Danny and Hume, Tristan and Johnston, Scott and Kravec, Shauna and Lovitt, Liane and Nanda, Neel and Olsson, Catherine and Amodei, Dario and Brown, Tom and Clark, Jack and McCandlish, Sam and Olah, Chris and Mann, Ben and Kaplan, Jared},
	month = apr,
	year = {2022},
	note = {arXiv:2204.05862 [cs]},
}

@misc{chen_decision_2021,
	title = {Decision {Transformer}: {Reinforcement} {Learning} via {Sequence} {Modeling}},
	shorttitle = {Decision {Transformer}},
	url = {http://arxiv.org/abs/2106.01345},
	doi = {10.48550/arXiv.2106.01345},
	urldate = {2023-12-31},
	publisher = {arXiv},
	author = {Chen, Lili and Lu, Kevin and Rajeswaran, Aravind and Lee, Kimin and Grover, Aditya and Laskin, Michael and Abbeel, Pieter and Srinivas, Aravind and Mordatch, Igor},
	month = jun,
	year = {2021},
	note = {arXiv:2106.01345 [cs]},
}

@misc{rafailov_direct_2023,
	title = {Direct {Preference} {Optimization}: {Your} {Language} {Model} is {Secretly} a {Reward} {Model}},
	shorttitle = {Direct {Preference} {Optimization}},
	url = {http://arxiv.org/abs/2305.18290},
	doi = {10.48550/arXiv.2305.18290},
	urldate = {2023-12-31},
	publisher = {arXiv},
	author = {Rafailov, Rafael and Sharma, Archit and Mitchell, Eric and Ermon, Stefano and Manning, Christopher D. and Finn, Chelsea},
	month = dec,
	year = {2023},
	note = {arXiv:2305.18290 [cs]},
}

@misc{korbak_rl_2022,
	title = {{RL} with {KL} penalties is better viewed as {Bayesian} inference},
	url = {http://arxiv.org/abs/2205.11275},
	doi = {10.48550/arXiv.2205.11275},
	urldate = {2023-12-31},
	publisher = {arXiv},
	author = {Korbak, Tomasz and Perez, Ethan and Buckley, Christopher L.},
	month = oct,
	year = {2022},
	note = {arXiv:2205.11275 [cs, stat]},
}

@misc{levine_reinforcement_2018,
	title = {Reinforcement {Learning} and {Control} as {Probabilistic} {Inference}: {Tutorial} and {Review}},
	shorttitle = {Reinforcement {Learning} and {Control} as {Probabilistic} {Inference}},
	url = {http://arxiv.org/abs/1805.00909},
	doi = {10.48550/arXiv.1805.00909},
	urldate = {2023-12-31},
	publisher = {arXiv},
	author = {Levine, Sergey},
	month = may,
	year = {2018},
	note = {arXiv:1805.00909 [cs, stat]},
}

@inproceedings{christiano_deep_2017,
	address = {Red Hook, NY, USA},
	series = {{NIPS}'17},
	title = {Deep reinforcement learning from human preferences},
	isbn = {978-1-5108-6096-4},
	urldate = {2023-09-28},
	booktitle = {Proceedings of the 31st {International} {Conference} on {Neural} {Information} {Processing} {Systems}},
	publisher = {Curran Associates Inc.},
	author = {Christiano, Paul F. and Leike, Jan and Brown, Tom B. and Martic, Miljan and Legg, Shane and Amodei, Dario},
	month = dec,
	year = {2017},
	pages = {4302--4310},
}

@inproceedings{haarnoja_reinforcement_2017,
	title = {Reinforcement {Learning} with {Deep} {Energy}-{Based} {Policies}},
	url = {https://proceedings.mlr.press/v70/haarnoja17a.html},
	language = {en},
	urldate = {2023-03-17},
	booktitle = {Proceedings of the 34th {International} {Conference} on {Machine} {Learning}},
	publisher = {PMLR},
	author = {Haarnoja, Tuomas and Tang, Haoran and Abbeel, Pieter and Levine, Sergey},
	month = jul,
	year = {2017},
	note = {ISSN: 2640-3498},
	pages = {1352--1361},
}

@article{brown_language_2020,
	title = {Language {Models} are {Few}-{Shot} {Learners}},
	url = {http://arxiv.org/abs/2005.14165},
	urldate = {2021-11-30},
	journal = {arXiv:2005.14165 [cs]},
	author = {Brown, Tom B. and Mann, Benjamin and Ryder, Nick and Subbiah, Melanie and Kaplan, Jared and Dhariwal, Prafulla and Neelakantan, Arvind and Shyam, Pranav and Sastry, Girish and Askell, Amanda and Agarwal, Sandhini and Herbert-Voss, Ariel and Krueger, Gretchen and Henighan, Tom and Child, Rewon and Ramesh, Aditya and Ziegler, Daniel M. and Wu, Jeffrey and Winter, Clemens and Hesse, Christopher and Chen, Mark and Sigler, Eric and Litwin, Mateusz and Gray, Scott and Chess, Benjamin and Clark, Jack and Berner, Christopher and McCandlish, Sam and Radford, Alec and Sutskever, Ilya and Amodei, Dario},
	month = jul,
	year = {2020},
	note = {arXiv: 2005.14165},
}

\newpage
\section*{Checklist}



\begin{enumerate}

  \item For all models and algorithms presented, check if you include:
  \begin{enumerate}
    \item A clear description of the mathematical setting, assumptions, algorithm, and/or model. [Yes]
    \item An analysis of the properties and complexity (time, space, sample size) of any algorithm. [Yes]
    \item (Optional) Anonymized source code, with specification of all dependencies, including external libraries. [Yes]
  \end{enumerate}

  \item For any theoretical claim, check if you include:
  \begin{enumerate}
    \item Statements of the full set of assumptions of all theoretical results. [Yes]
    \item Complete proofs of all theoretical results. [Yes]
    \item Clear explanations of any assumptions. [Yes]
  \end{enumerate}

  \item For all figures and tables that present empirical results, check if you include:
  \begin{enumerate}
    \item The code, data, and instructions needed to reproduce the main experimental results (either in the supplemental material or as a URL). [Yes]
    \item All the training details (e.g., data splits, hyperparameters, how they were chosen). [Not Applicable]
    \item A clear definition of the specific measure or statistics and error bars (e.g., with respect to the random seed after running experiments multiple times). [Yes]
    \item A description of the computing infrastructure used. (e.g., type of GPUs, internal cluster, or cloud provider). [Not Applicable]
  \end{enumerate}

  \item If you are using existing assets (e.g., code, data, models) or curating/releasing new assets, check if you include:
  \begin{enumerate}
    \item Citations of the creator If your work uses existing assets. [Not Applicable]
    \item The license information of the assets, if applicable. [Not Applicable]
    \item New assets either in the supplemental material or as a URL, if applicable. [Not Applicable]
    \item Information about consent from data providers/curators. [Not Applicable]
    \item Discussion of sensible content if applicable, e.g., personally identifiable information or offensive content. [Not Applicable]
  \end{enumerate}

  \item If you used crowdsourcing or conducted research with human subjects, check if you include:
  \begin{enumerate}
    \item The full text of instructions given to participants and screenshots. [Not Applicable]
    \item Descriptions of potential participant risks, with links to Institutional Review Board (IRB) approvals if applicable. [Not Applicable]
    \item The estimated hourly wage paid to participants and the total amount spent on participant compensation. [Not Applicable]
  \end{enumerate}

\end{enumerate}

\clearpage
\appendix
\thispagestyle{empty}

\onecolumn

\section{PROOFS}\label{appendix:proofs}

\subsection{Proofs for~\Cref{sec:incoherence}}

\begin{proposition}[Characterisation of $V$ and $Q$]
    For any prior $\pi(a|s)$ and any non-positive reward function $r(a,s)$, we have simple expressions for the soft $Q$ and $V$ functions given by:
    \begin{align*}
        Q^\pi(a_t, s_t) &= \log p_\pi(\OO_{t:T} = 1|s_t,a_t) \\
        V^\pi(s_t) &= \log p_\pi(\OO_{t:T} = 1|s_t)
    \end{align*}
    where we define the auxiliary \emph{optimality variables} $\OO_t$ to have Bernoulli distributions with:
    \[
    p_\pi(\OO_t = 1|s_t,a_t) = e^{r(s_t,a_t)}
    \]
\end{proposition}
\begin{proof}
    We prove this by backward induction. Base case $t = T$, we have:
    \begin{align*}
    Q^\pi(s_T, a_T) &= r(s_T, a_T) = \log (\exp r(s_T, a_T)) \\
    &= \log p_\pi(\OO_{T} = 1|s_T,a_T)
    \end{align*}
    and:
    \begin{align*}
    V^\pi(s_T) &= \log \Expect{a_t \sim \pi(a_T|s_T)}{\exp Q^\pi(s_t, a_T)} \\
    &= \log \Expect{a_t \sim \pi(a_T|s_T)}{p_\pi(\OO_{T} =1|s_T, a_T})\\
    &= \log p_\pi(\OO_{T}=1|s_T)
    \end{align*}
    
    Now, assuming that the hypothesis holds for $t < t' \leq T$, we compute for $t$:
    \begin{align*}
        Q^\pi(&s_{t},a_{t}) = r(s_{t}, a_{t}) + \log \Expect{s_{t+1}}{\exp(V^\pi(s_{t+1})}) \\
        &= \log (p_\pi(\OO_{t} = 1|s_{t},a_{t}) \cdot \Expect{s_{t+1}}{p_\pi(\OO_{t+1:T}|s_{t+1})}) \\
        &= \log \left(p_\pi(\OO_{t} = 1|s_{t},a_{t})  \cdot p_\pi(\OO_{t+1:T}|s_{t+1}, a_{t+1})\right) \\
        &= \log p_\pi(\OO_{t:T} = 1|s_{t},a_{t})
    \end{align*}
    and the proof of the inductive step for $V^\pi(s_t, a_t)$ is the same as the base case shown above.
\end{proof}

\begin{proposition}
    In case of deterministic dynamics, we have that $p(\xi|\OO_{1:T}) = \hat{p}(\xi)$.
\end{proposition}
\begin{proof}
We compute the policy for time steps $t$ and $t+1$:
\begin{align*}
p(a_t|s_t, \OO_{t:T}) &= \frac{p(\OO_{t:T}|a_t, s_t)p(a_t|s_t)}{p(\OO_{t:T}|s_t)} \\
&= \frac{p(\OO_{t:T}|s_{t+1})p(a_t|s_t)}{p(\OO_{t:T}|s_t)}    
\end{align*}
and
\[
p(a_{t+1}|s_{t+1}, \OO_{t+1:T}) = \frac{p(\OO_{t+1:T}|a_{t+1}, s_{t+1})p(a_{t+1}|s_{t+1})}{p(\OO_{t+1:T}|s_{t+1})}
\]
where $\tau(s_{t+1}|s_{t}, a_{t}) = 1$. Thus, considering $\log \hat{p}(\xi)$ we have telescoping series:
\begin{align*}
\log \hat{p}(\xi) &= \sum_{t = 1}^{T} \log p(\OO_{t:T}|s_{t}, a_{t}) - \log p(\OO_{t:T}|s_{t}) + \log p(a_t|s_t) \\
=& \log(\OO_{T:T}|a_T, s_T) - \log p(\OO_{1:T}|s_1) + \sum_{t = 1}^T \log p(a_t|s_t)
\end{align*}
which is exactly:
\[
\log p(\xi|\OO_{1:T}) = \log p(\OO_{1:T}|\xi) - \log(\OO_{1:T}) + \log p(\xi)
\]
\end{proof}
\begin{proposition}
    Given deterministic dynamics $\tau$ and policy $\pi(a|s)$ and an order-respecting $f: \mathbb{R}^{|A|} \to \Delta(|A|)$, such that for $x$ having a unique maximum, $f(x)$ is an $\arg\max$ indicator, $\pi$ is $f$-coherent if and only if it is greedy w.r.t. its own soft‑$Q$ function.
\end{proposition}
\begin{proof}

We first show argmax necessity under order‑respecting $f$: let $x\in\mathbb{R}^{|A|}$ have a unique maximizer $m=\arg\max_i x_i$. If $f$ is order‑respecting and outputs a point mass at $x$, then $f(x)=\delta_m$. Indeed, suppose $f(x)=\delta_i$ with $i\neq m$. Then $x_m>x_i$, but the last condition in the order‑preservation definition demands $f_m(x)\ge f_i(x)=1$, a contradiction.

Now, the forward direction. Assume $\pi$ is $f$-coherent, and fix the state $s$. If $Q^\pi(s,\cdot)$ has a unique maximizer, from what we proved above we know that $f(Q^\pi(s,\cdot)) = \delta_{\arg\max}$, i.e. $\pi(s)$ must be an argmax. If there are ties, coherence still forces $\pi(s)$ to be one of the maximizers, which is consistent with the selector’s tie‑break, as otherwise the per‑state $\KL$ contributes $\infty$. So $\pi(s)\in\arg\max Q^\pi(s,\cdot)$.

For the reverse direction, if $\pi(s)\in\arg\max_a Q^\pi(s,a))$ for all $s$, we pick $f$ to be the $\arg\max$ selector $f(Q_\pi(s,\cdot))=\delta_{\pi(s)}$ for all $s$. Hence $\KL(\pi||\pi^{Q,f})=0$ by~\Cref{prop:alternative-incoherence}, i.e., $\pi$ is $f$-coherent.

\end{proof}


\begin{proposition}
    The policy $\pi^\BB_T$ is $f$-coherent.
\end{proposition}
\begin{proof}
    We prove by backwards induction that $\pi(s_t|a_t)$ is fixed in iteration $t$ and does not change afterwards. In the base case $t = T$, the $Q$ function $Q^\pi(s_T, \cdot)$ does not depend on $\pi$, so it is fixed in the first iteration and remains unchanged. The inductive case follows because $\pi(s_t|a_t)$ only depends on future times $t' > t$.
\end{proof}

\begin{proposition}
    Given any prior $\pi(a|s)$, there exists a policy $\pi^*(a|s)$ such that we have convergence in distribution:
    \[
        \lim_{\delta \to 0} \pi^\delta \to \pi^*
    \]
    where $\pi^\delta$ is defined with respect to the soft Q function induced by the prior $\pi$. Moreover, the policy $\pi^*$ can be explicitly characterised as the uniform distribution over $A^* := \{a^* \in \A: Q^\pi(s, a^*) = \max_{a \in \A} Q^\pi(s, a)\}$.
\end{proposition}
\begin{proof}
    Let us denote $A_t^* := \{a^* \in \A: Q^\pi(s_t, a^*) = \max_{a \in \A} Q^\pi(s_t, a)\}$ and $a_t^* \in A^*$ arbitrarily chosen. \\
    For any $1 \leq t \leq T$ we have:
    \begin{align*}
    \lim_{\delta \to 0} \pi^\delta(a_t^*|s_t) &= \lim_{\delta \to 0} \frac{\exp\left(\frac{1}{\delta} Q^\pi(s_t, a_t^*)\right)}{\sum_{a \in A} \exp \left(\frac{1}{\delta} Q^\pi(s_t, a)\right)} \\
    &= \lim_{\delta \to 0} \frac{\exp\left(\frac{1}{\delta} Q^\pi(s_t, a^*) - Q^\pi(s_t, a^*)\right)}{\sum_{a \in A} \exp \left(\frac{1}{\delta} Q^\pi(s_t, a) - Q^\pi(s_t, a^*)\right)} \\
    &= \lim_{\delta \to 0} \frac{1}{\sum_{a \in A} \exp \left(\frac{1}{\delta} Q^\pi(s_t, a) - Q^\pi(s_t, a^*)\right)}
    \end{align*}
    where the denominator now splits into:
    \begin{align*}
    &\lim_{\delta \to 0} \sum_{a \in A_t^*} \exp \left(\frac{1}{\delta} Q^\pi(s_t, a) - Q^\pi(s_t, a^*)\right) = |A^*_t|  \\
    &\lim_{\delta \to 0} \sum_{a \in A\backslash A_t^*} \exp \left(\frac{1}{\delta} Q^\pi(s_t, a) - Q^\pi(s_t, a^*)\right) = 0
    \end{align*}
    giving $\pi^\delta(a|s_t)$ uniform on $A_t^*$.
\end{proof}

\begin{corollary}
    A necessary condition for policy $\pi(a|s)$ to be optimal is that:
    \[
    \lim_{\delta \to 0} \kappa_\delta(\pi) < \infty
    \]
\end{corollary}
\begin{proof}
    From~\Cref{def:boltzmann-incoherence},~\Cref{prop:alternative-incoherence} and ~\Cref{lemma:optimal-maxent}, we write:
    \begin{align*}
    \lim_{\delta \to 0} \kappa_\delta(\pi) &= \lim_{\delta \to 0} \KL_\tau(\pi(\tau)|\pi^\delta(\tau)) \\
    &= \lim_{\delta \to 0} \sum_{t = 1}^T \mathbb{E}_{s_t \sim d^t_\pi} \KL(\pi(\cdot|s_t) || \pi^\delta(\cdot|s_t)) \\
    &= \sum_{t = 1}^T \mathbb{E}_{s_t \sim d^t_\pi} \KL(\pi(\cdot|s_t) || \lim_{\delta \to 0} \pi^\delta(\cdot|s_t)) \\
    &= \sum_{t = 1}^T \mathbb{E}_{s_t \sim d^t_\pi} \KL(\pi(\cdot|s_t) || \pi^*(\cdot|s_t))
    \end{align*}
    Thus, for any $s_t \sim d^t_\pi$, the $\KL(\pi|| \pi^*)$ is finite if and only $\text{supp}(\pi^*) \subseteq \text{supp}(\pi)$. Meaning that $\pi$ must assign positive mass to every action in $A_t^*$ for each $t$. But this is a necessary condition for optimality (by backwards induction on $t$).
\end{proof}

\subsection{Proofs for~\Cref{sec:removing-incoherence}}

\begin{proposition}[Strong return improvement lemma]
    The sequence of policies $(\pi^\GG_i)_{i = 0, 1, \dots}$ given by the control-by-inference improves its return monotonically, that is:
    \[
    J(\pi^\GG_{i + 1}) \geq J(\pi^\GG_{i}) 
    \]
\end{proposition}
\begin{proof}    
    We prove this by backwards induction on $T$. First, fix some index $i$ and denote for brevity:
    \[
    \pi = \pi^\GG_i \qquad \pi' = \pi^\GG_{i+1}
    \]
    We also define the future probability of success in a state $s_t$ as in~\Cref{def:softvq}:
    \begin{align*}
    V^\pi(s_t) &= \log p_\pi(\mathcal{O}_{t:T} = 1| s_t) \\
    Q^\pi(s_t, a_t) &= \log  p_\pi(\mathcal{O}_{t:T} = 1| s_t, a_t)
    \end{align*}
    Now, we proceed by induction to show that:
    \[
    V^{\pi'}(s_t) \geq V^{\pi}(s_t)
    \]
    The policy $\pi' = \mathcal{G}(\pi)$ can be written from Bayes theorem as:
    \begin{equation}\label{eq:posterior}
    \pi'(a_t|s_t) = \frac{\pi(a_t|s_t) \exp Q^\pi(a_t|s_t)}{\mathbb{E}_{a_t \sim \pi(\cdot|s_t)}[\exp Q^\pi(a|s_t)]}    
    \end{equation}
    Now:
    \begin{equation}\label{eq:inequality}
    \begin{aligned}
        \mathbb{E}_{a_t \sim \pi'(\cdot|s_t)}[\exp Q^\pi(a|s_t)]
        &= \frac{\mathbb{E}_{a_t \sim \pi(\cdot|s_t)}\left[\left(\exp Q^\pi(a|s_t)\right)^2\right]}{\mathbb{E}_{a_t \sim \pi(\cdot|s_t)}[\exp Q^\pi(a|s_t)]} \\
        & \geq \mathbb{E}_{a_t \sim \pi(\cdot|s_t)}[\exp Q^\pi(a|s_t)]
    \end{aligned}
    \end{equation}
    where the first line follows from~\Cref{eq:posterior}, and the second line follows from Cauchy–Schwarz inequality.
    
    First, assume $t=T$. Then, $Q^\pi(a_T|s_T)$ does not depend on $\pi$ and is simply given by $r(s_t, a_t)$. We then immediately have from~\Cref{eq:inequality} that:
    \[
    V^{\pi'}(\mathcal{O}_{T:T}|s_t) \geq V^{\pi}(\mathcal{O}_{T:T}|s_t)
    \]
    establishing the base of the induction. For the inductive case $t < T$, assume that the above holds for all $t < t' \leq T$ and write:
    \begin{align*}
      \exp V^{\pi'}(s_t) &= p_{\pi'}(\mathcal{O}_{t:T}=1 | s_t)\\
      &= \mathbb{E}_{a_t\sim \pi'(\cdot| s_t)}\left[e^{r(s_t,a_t)}
      \mathbb{E}_{s_{t+1}} \exp V^{\pi'}(s_{t+1})\right] \\
      &\geq \mathbb{E}_{a_t\sim \pi'(\cdot| s_t)}\left[e^{r(s_t,a_t)}
      \mathbb{E}_{s_{t+1}} \exp V^{\pi}(s_{t+1})\right] \\
      &= \mathbb{E}_{a_t\sim \pi'(\cdot| s_t)}\left[\exp Q^\pi(s_t, a_t)\right] \\
      &\geq \mathbb{E}_{a_t\sim \pi(\cdot| s_t)}\left[\exp Q^\pi(s_t, a_t)\right] \\
      &= \exp V^\pi(s_t)
    \end{align*}
    where the first inequality comes from the inductive hypothesis, and the second from~\Cref{eq:inequality}.
    From the Cauchy–Schwarz, inequalities are strict unless $Q^\pi(s_t,\cdot)$ is $\pi(\cdot| s_t)$-almost-surely constant. Intuitively, the update is non‑trivial exactly when the old policy is not already proportional to the frozen‑future success factors.
\end{proof}

\subsection{Proof of~\Cref{sec:equivalence}}

We split the proof of~\Cref{theorem:equivalence-thrice} into several lemmas.

\begin{lemma}~\label{lemma:equivalence-temperature-folding}
    For any prior $p(a|s)$, non-positive reward function $r(s_t,a_t)$, deterministic dynamics $\tau(s_{t+1}|s_t, a_t)$, and $k \in \mathbb{N}_+$ we have $\pi_{\alpha(k)} = \pi^\HH_k$, for $\alpha(k) = 2^k$.
\end{lemma}
\begin{proof}
    Throughout the proof, we will assume that the prior is uniform; otherwise the prior is folded into the reward. Let us first consider the case of $T = 1$. We have that:
    \begin{align*}
        \log p(a|s, \mathcal{O}^{(k)}) &= \log\left(\frac{p(\mathcal{O}^{(k)}|a, s)p(a|s)}{p(\mathcal{O}^{(k)}|s)}\right) \\
        &= r_k(s, a) + C
    \end{align*}
    where the constant
    \[
    C = \log |A| - \log \frac{\sum_{a}\exp r_k(s, a)}{|A|}
    \]
    does not depend on $a$. We also observe that shifting reward by any additive constant does not change the policy (a soft version of the reward shaping lemma): for $r'(s, a) = r(s,a) + C$ we have
    \begin{align*}
        p(a|s, \mathcal{O}') &= \frac{p(a|s) e^{r(s, a) + C}}{\sum_{a'} p(a'|s)e^{r(s,a) + C}} \\
        &= \frac{p(a|s) e^{r(s, a)}}{\sum_{a'} p(a'|s)e^{r(s,a)}}
    \end{align*}
    To ignore the constants, we write $r \equiv r'$ when $r = r' + C$. Thus, the iterative process of adding the posterior to the reward results in a sequence:
    \begin{align*}
    r_0(s, a) &\equiv r(s, a) \\
    r_1(s, a) &\equiv r_0(s, a) + \log p(a|s, \mathcal{O}^{0}) \\
    &\equiv 2r(s, a) + C_1 \equiv 2r(s,a) \\
    \dots\\
    r_{k}(s, a) &\equiv r_{k-1}(s, a) + \log p(a|s, \mathcal{O}^{(k-1)}) \\
    &\equiv 2r_{k-1}(s, a) + C_k \equiv 2^{k}r(s,a)
    \end{align*}
    This proves, by induction on $k$, that $\pi^\HH_k = \pi_{\alpha(k)}$.\\
    The case of $T > 1$ is then done by backwards induction over the time horizon. We are given $1 \leq t < T$ and we assume that the following inductive hypothesis holds for all $t < t'$ and for all $k$:
    \[
    r_k(s, a) = 2^kr_0(s, a) \qquad p(\mathcal{O}_{t':T}^{\alpha(k)}|s_{t'}) = p(\mathcal{O}_{t':T}^{(k)}|s_{t'})
    \]
    In case of the policy obtained by decreasing the temperature, we have:
    \begin{align*}
        \pi_{\alpha(k+1)}(a_t|s_t) &= \log p(a_t|s_t, \mathcal{O}_{t:T}^{\alpha(k+1)}) \\
        & \equiv \log p(\mathcal{O}_{t:T}^{\alpha(k+1)}|a_t, s_t) + \log p(a_t|s_t) - \log p(\OO_{t:T}|s_t)\\
        & \equiv \log p(\mathcal{O}_{t}^{\alpha(k+1)}|a_t, s_t) + \log p(\mathcal{O}_{t+1:T}^{\alpha(k+1)}|a_t, s_t) \\
        &= 2^{k + 1}r_0(a_t, s_t) + \log p(\mathcal{O}_{t+1:T}^{\alpha(k+1)}|s_{t+1})
    \end{align*}
    where the first line follows from the definition, the next one from the Bayes' law, the next one from the assumption that the prior is uniform, and the next one from the definition of the $\mathcal{O}_{1}^{\alpha(k+1)}$ and using the fact that there is a unique $s_{t+1}$ which follows $\tau(a_t, s_t)$.
    
    We look at the policy obtained by folding the posterior into the reward:
    \begin{align*}
        \pi_{(k+1)}^\HH(a_t|s_t) &= \log p(a_t|s_t, \mathcal{O}_{t:T}^{(k+1)}) \\
        & \equiv \log p(\mathcal{O}_{t}^{(k+1)}|a_t, s_t) + \log p(\mathcal{O}_{t+1:T}^{(k+1)}|a_t, s_t) \\
        &= r_{k+1}(a_t, s_t) + \log p(\mathcal{O}_{t+1:T}^{(k+1)}|s_{t+1}) \\
        &= 2^{k+1}r_0(a_t, s_t) + \log p(\mathcal{O}_{t+1:T}^{\alpha(k+1)}|s_{t+1})
    \end{align*}
    where the last line follows from the inductive hypothesis.
\end{proof}

\begin{lemma}~\label{lemma:equivalence-folding-retraining}
    For any prior $p(a|s)$, non-positive reward function $r(s_t,a_t)$, arbitrary dynamics $\tau(s_{t+1}|s_t, a_t)$, and $k \in \mathbb{N}_+$ we have $\pi^\mathcal{G}_{k} = \pi^\FF_k$.
\end{lemma}
\begin{proof}
Proof by induction on $k$. Assume without loss of generality uniform prior $p$, otherwise fold it into the reward. For $k=0$ they coincide by definition. Now, assume that $\pi_k^\GG = \pi_k^\FF$, and we show the inductive step for $k+1$. Let us denote the corresponding $Q$ functions by $Q^\GG_k$ and $Q^\FF_k$.
First, exactly as in the proof of~\Cref{prop:strong-improvement}, we know that:
\[
        \pi_{k+1}^\GG(a_t|s_t) = \frac{\pi_{k}^\GG(a_t|s_t) \exp Q^{\GG}_k(a_t|s_t)}{\mathbb{E}_{a_t \sim \pi(\cdot|s_t)}[\exp Q^{\GG}_k(a|s_t)]}\ \propto\ \pi_{k}^\GG(a_t|s_t) \exp Q^{\GG}(a_t|s_t)
\]
and on the other hand, we also know from the definition that:
\[
    \pi_{k+1}^\FF(a_t|s_t) \ \propto \ p(s_t|a_t) \exp Q^{\FF}_k(a_t|s_t)
\]
Thus, we aim to show that: \[
p(s_t|a_t) \exp Q^{\FF}_k(a_t|s_t) = C_t\pi_{k}^\GG(a_t|s_t) \exp Q^{\GG}_k(a_t|s_t)
\]
up to some multiplicative constant $C_t$. We do this again by induction, this time on a time horizon $t$.
Using~\Cref{def:softvq}, we have:
\[
  \exp Q^{\FF}_k(s_T,a_T)=\exp r_k(s_T,a_T)
  =\pi^\GG_k(a_T| s_T)\exp r_0(s_T,a_T)
  =\pi^\GG_k(a_T| s_T)\exp Q^{\GG}_k(s_T,a_T)
\]
which proves this in case $t=T$. Inductive step: assume that the hypothesis holds for $t+1$, and write:
\[
\exp V^\FF(s_{t+1})
=\mathbb{E}_{a_{t+1}\sim p}[\exp Q^{\FF}_k(s_{t+1}(a_{t+1})]
=\frac{C_t}{|A|}\mathbb{E}_{a_{t+1}\sim \pi^\GG_k}[\exp Q^\GG_k(s_{t+1}(a_{t+1})]
\]
Taking $\log$ and bringing the factor outside the expectation shows that:
 \[
  \mathbb{E}_{s_{t+1}}[\exp V^\FF_k(s_{t+1})]
  =\frac{C_t}{|A|}\mathbb{E}_{s_{t+1}}[\exp V^\GG_k(s_{t+1})]
\]
which means:
\[
\exp Q^\FF_k(s_t,a_t)
  =\exp r_k(s_t,a_t)\cdot \frac{C_t}{|A|}
  \mathbb{E}_{s_{t+1}}[\exp V^\GG_k(s_{t+1})]
\]
proving the lemma.
\end{proof}

\begin{lemma}~\label{lemma:equivalence-folding-orig}
    For any prior $p(a|s)$, non-positive reward function $r(s_t,a_t)$, deterministic dynamics $\tau(s_{t+1}|s_t, a_t)$, and $k \in \mathbb{N}_+$ we have $\pi^\FF_{k} = \pi_{\alpha(k)}$ for schedule $\alpha(k) = k$.
\end{lemma}
\begin{proof}
 The proof follows almost exactly like the one in~\Cref{lemma:equivalence-temperature-folding}. The base case is identical. In the inductive case, we write analogously:
\begin{align*}
    r_0(s, a) &\equiv r(s, a) \\
    r_1(s, a) &\equiv r_0(s, a) + \log p(a|s, \mathcal{O}^{0}) \\
    &\equiv 2r(s, a) + C_1 \equiv 2r(s,a) \\
    \dots\\
    r_{k}(s, a) &\equiv r_0(s, a) + \log p(a|s, \mathcal{O}^{(k-1)}) \\
    &\equiv kr(s, a) + C_k \equiv kr(s,a)
    \end{align*}
\end{proof}
and the rest of the proof follows as before.

\subsection{Proof of~\Cref{corr:rate-of-convergence}}

\begin{lemma}~\label{lemma:argmax-derivative}
Given a function $h(x) = f(\arg\max_t (f(t) + x g(t)))$ for some differentiable functions $f, g$, the derivative of $h(x)$ is:
    \[
        \frac{dh}{dx} = -\frac{f'(t^*(x)) g'(t^*(x))}{f''(t^*(x)) + x g''(t^*(x))}
    \]
    where $t^*(x) = \arg\max_t f(t) + x g(t)$.
\end{lemma}
\begin{proof}
Let $t^*(x) = \arg\max_t (f(t) + x g(t))$. Therefore, we have:
\[
    \frac{d}{dt} \left( f(t) + x g(t) \right) \bigg|_{t = t^*(x)} = 0 = f'(t^*(x)) + x g'(t^*(x))
\]
Thus:
\[ x = -\frac{f'(t^*(x))}{g'(t^*(x))} \]

To differentiate \( h(x) \), we use the fact that \( h(x) = f(t^*(x)) \). This gives us:
\[ \frac{dh}{dx} = \frac{d}{dx} f(t^*(x)) \]

By the chain rule:
\[ \frac{dh}{dx} = f'(t^*(x)) \cdot \frac{dt^*(x)}{dx} \]

To find \( \frac{dt^*(x)}{dx} \), we differentiate the first-order condition \( f'(t^*(x)) + x g'(t^*(x)) = 0 \) with respect to \( x \):
\[ \frac{d}{dx} \left( f'(t^*(x)) + x g'(t^*(x)) \right) = 0 \]

Applying the chain rule:
\[ f''(t^*(x)) \cdot \frac{dt^*(x)}{dx} + g'(t^*(x)) + x g''(t^*(x)) \cdot \frac{dt^*(x)}{dx} = 0 \]

Rearrange to solve for \( \frac{dt^*(x)}{dx} \):
\[ \left( f''(t^*(x)) + x g''(t^*(x)) \right) \frac{dt^*(x)}{dx} = -g'(t^*(x)) \]
\[ \frac{dt^*(x)}{dx} = -\frac{g'(t^*(x))}{f''(t^*(x)) + x g''(t^*(x))} \]

Finally, substituting this back into the expression for \( \frac{dh}{dx} \):
\[ \frac{dh}{dx} = f'(t^*(x)) \cdot \left( -\frac{g'(t^*(x))}{f''(t^*(x)) + x g''(t^*(x))} \right) \]

\end{proof}

\begin{corollary}[Return improvement rate]
    In case of deterministic dynamics $\tau$, given a sequence of policies $\pi^\GG_i$, for $k \in \mathbb{N}_+$ have that:
    \[
    J(\pi^\GG_{k}) - J(\pi^\GG_{k-1}) = \frac{1}{{k}} \cdot \frac{J'(\pi^\GG_{k}) \cdot \hat{\HH}'(\pi^\GG_{k})}{J''(\pi^\GG_k) + \frac{1}{{k}} \hat{\HH}''(\pi^\GG_{k})} + O\left(\frac{1}{{k}^2}\right)
    \]
    where $\hat{\HH}(\pi)$ denotes the causal entropy of policy $\pi$, and the derivatives are taken with respect to the temperature $\alpha$.
\end{corollary}
\begin{proof}
    To prove this, we use an alternative characterization of the maximum entropy  policies~\citep{haarnoja_reinforcement_2017,levine_reinforcement_2018}. For a fixed temperature $\alpha \in \RR_+$, the policy $\pi_\alpha$ maximizes the functional:
    \[
    J_\alpha(\pi) = \mathbb{E}_{s_t, a_t}[r(s_t, a_t) + \frac{1}{\alpha} \mathcal{H}(\pi(\cdot|s_t))]= J(\pi) + \frac{1}{\alpha} \hat{\HH}(\pi)
    \]
    where the expectation is taken over the trajectory induced by the policy $\pi$, $\HH(\pi(\cdot|s))$ denotes the Shannon entropy of the policy in state $s$, and $\hat{\HH}$ denotes the causal entropy of the policy $\pi$. By~\Cref{lemma:equivalence-folding-orig} we know that $\pi^\GG_j = \pi_{\alpha(j)}$ for $\alpha(j) = j$.

    From the Taylor expansion, for any twice-differentiable $u(\alpha)$ we have:
    \[
    u(\alpha + \eta) - u(\alpha) = \eta u'(\alpha) + O(\eta^2)
    \]
    Substituting $u(\alpha) = J(\arg\max_\pi J_{\frac{1}{\alpha}}(\pi))$, we obtain: 
    \begin{align*}
        J(\pi^\GG_{k-1}) - J(\pi^\GG_{k})
        &= J(\pi_{\alpha(k - 1)}) - J(\pi_{\alpha(k)}) \\
        &= u\left(\frac{1}{{k}} + \eta\right) - u\left(\frac{1}{{k}}\right)  \\
        &= \eta u'\left(\frac{1}{{k}}\right) + O\left(\eta^2\right)
    \end{align*}
    for $\eta = \frac{1}{k({k-1})}$. Now, using~\Cref{lemma:argmax-derivative}, we derive:
    \[
    u'(x) = -\frac{J'(\pi^*(x)) \hat{\HH}'(\pi^*(x))}{J''(\pi^*(x)) + x \hat{\HH}''(\pi^*(x))}
    \]
    and substituting $x = \frac{1}{{k}}$, and therefore $\pi^*(x) = \pi_{\alpha(k)} = \pi^\GG_k$, we get:
    \[
    u'(x) = -\frac{J'(\pi_{\alpha(k)}) \hat{\HH}'(\pi_{\alpha(k)})}{J''(\pi^\GG_k) + \frac{1}{{k}} \hat{\HH}''(\pi_{\alpha(k)})}
    \]
    which gives us the final equation:
    \[
    J(\pi^\GG_{k}) - J(\pi^\GG_{k-1}) = \frac{1}{{k}} \frac{J'(\pi^\GG_{k}) \hat{\HH}'(\pi^\GG_{k})}{J''(\pi^\GG_{k})) + \frac{1}{{k}} \hat{\HH}''(\pi^\GG_{k})} + O\left(\frac{1}{k}^2\right)
    \]
\end{proof}

\section{EXAMPLE ~\ref{example:counter} EXPLICIT CALCULATION }

Calculation for the~\Cref{example:counter} in detail.\\
We take a three-state, two action Markov decision process with initial state $\emptyset$ and a uniform prior policy:
\[
\pi(a_1|\emptyset) = \pi(a_2|\emptyset) = \frac{1}{2}
\]
with transition dynamics depicted on the diagram below:
\[\begin{tikzcd}
	{s_1} & \emptyset & {s_2}
	\arrow["{\frac{3}{4}}"', color={rgb,255:red,0;green,108;blue,238}, curve={height=12pt}, dashed, from=1-2, to=1-1]
	\arrow["{\frac{1}{2}}", curve={height=-12pt}, from=1-2, to=1-1]
	\arrow["{\frac{1}{4}}", color={rgb,255:red,0;green,108;blue,238}, curve={height=-12pt}, dashed, from=1-2, to=1-3]
	\arrow["{\frac{1}{2}}"', curve={height=12pt}, from=1-2, to=1-3]
\end{tikzcd}\]
where red dotted lines correspond to action $a_1$, and solid black lines to action $a_2$. We also assume that:
\[
r(s_0, \cdot) = \log 1 \qquad r(s_1, \cdot) = \log \frac{1}{3} \qquad r(s_2, \cdot) = \log \frac{2}{3}
\]
\subsection{Control-as-inference operator}
The first iteration of $\pi^\mathcal{G}_1(\cdot|\emptyset)$ computes the probability of getting reward for action $a_1$:
\[
\pi^\mathcal{G}_1(a_1|\emptyset) = \mathbb{P}(a_1|\mathcal{O}^2_{1:2}, \emptyset) = \frac{\mathbb{P}(\mathcal{O}^2_{1:2}|a_1, \emptyset)\mathbb{P}(a_1|\emptyset)}{\mathbb{P}(\mathcal{O}^2_{1:2})}
= \left(\frac{1}{2}\cdot \frac{3}{4}\cdot\frac{1}{3} + \frac{1}{2}\cdot \frac{1}{4}\cdot\frac{2}{3}\right)/{\mathbb{P}(\mathcal{O}^2_{1:2})}
= \frac{5}{24} / {\mathbb{P}(\mathcal{O}^2_{1:2})}
\]
and for $a_2$:
\[
\pi^\mathcal{G}_1(a_2|\emptyset) = \mathbb{P}(a_2|\mathcal{O}^2_{1:2}, \emptyset) = \frac{\mathbb{P}(\mathcal{O}^2_{1:2}|a_2, \emptyset)\mathbb{P}(a_2|\emptyset)}{\mathbb{P}(\mathcal{O}^2_{1:2})}
= \left(\frac{1}{2}\cdot \frac{1}{2}\cdot\frac{1}{3} + \frac{1}{2}\cdot \frac{1}{2}\cdot\frac{2}{3}\right)/{\mathbb{P}(\mathcal{O}^2_{1:2})}
= \frac{6}{24} / {\mathbb{P}(\mathcal{O}^2_{1:2})}
\]
After normalisation, we get:
\[
\pi^\mathcal{G}_1(a_1|\emptyset) = \frac{5}{11} \qquad \pi^\mathcal{G}_1(a_2|\emptyset) = \frac{6}{11}
\]
Now, applying the same computation second time - for $a_1$:
\[
\pi^\mathcal{G}_2(a_1|\emptyset) = \mathbb{P}(a_1|\mathcal{O}^2_{1:2}, \emptyset) = \frac{\mathbb{P}(\mathcal{O}^2_{1:2}|a_1, \emptyset)\pi^\mathcal{G}_1(a_1|\emptyset)}{\mathbb{P}(\mathcal{O}^2_{1:2})}
= \left(\frac{5}{11}\cdot \frac{3}{4}\cdot\frac{1}{3} + \frac{5}{11}\cdot \frac{1}{4}\cdot\frac{2}{3}\right)/{\mathbb{P}(\mathcal{O}^2_{1:2})}
= \frac{25}{11\cdot 12} / {\mathbb{P}(\mathcal{O}^2_{1:2})}
\]
and for $a_2$:
\[
\pi^\mathcal{G}_2(a_2|\emptyset) = \mathbb{P}(a_1|\mathcal{O}^2_{1:2}, \emptyset) = \frac{\mathbb{P}(\mathcal{O}^2_{1:2}|a_2, \emptyset)\pi^\mathcal{G}_1(a_2|\emptyset)}{\mathbb{P}(\mathcal{O}^2_{1:2})}
= \left(\frac{6}{11}\cdot \frac{3}{4}\cdot\frac{1}{3} + \frac{6}{11}\cdot \frac{1}{4}\cdot\frac{2}{3}\right)/{\mathbb{P}(\mathcal{O}^2_{1:2})}
= \frac{36}{11\cdot 12} / {\mathbb{P}(\mathcal{O}^2_{1:2})}
\]
Again applying normalisation:
\[
\pi^\mathcal{G}_2(a_1|\emptyset) = \frac{25}{61} \qquad \pi^\mathcal{G}_2(a_2|\emptyset) = \frac{36}{61}
\]
\subsection{Raising temperature}
Raising temperature policy $\pi_{\alpha(2)}(\cdot|\emptyset)$ first modifies the MDP by setting the rewards:
\[
r_2(s_1) = \log \frac{1}{9} \qquad r_2(s_2) = \log \frac{4}{9}
\]
and then recomputes the posterior for $a_1$:
\[
\pi_{\alpha(2)}(a_1|\emptyset) = \mathbb{P}(a_1|\mathcal{O}^2_{1:2}, \emptyset) = \frac{\mathbb{P}(\mathcal{O}^2_{1:2}|a_1, \emptyset)\mathbb{P}(a_1|\emptyset)}{\mathbb{P}(\mathcal{O}^2_{1:2})}
= \left(\frac{1}{2}\cdot \frac{3}{4}\cdot\frac{1}{9} + \frac{1}{2}\cdot \frac{1}{4}\cdot\frac{4}{9}\right)/{\mathbb{P}(\mathcal{O}^2_{1:2})}
= \frac{7}{72} / {\mathbb{P}(\mathcal{O}^2_{1:2})}
\]
and $a_2$:
\[
\pi_{\alpha(2)}(a_2|\emptyset) = \mathbb{P}(a_1|\mathcal{O}^2_{1:2}, \emptyset) = \frac{\mathbb{P}(\mathcal{O}^2_{1:2}|a_2, \emptyset)\mathbb{P}(a_2|\emptyset)}{\mathbb{P}(\mathcal{O}^2_{1:2})}
= \left(\frac{1}{2}\cdot \frac{1}{2}\cdot\frac{1}{9} + \frac{1}{2}\cdot \frac{1}{2}\cdot\frac{4}{9}\right)/{\mathbb{P}(\mathcal{O}^2_{1:2})} = \frac{10}{72}/{\mathbb{P}(\mathcal{O}^2_{1:2})}
\]
Again applying normalisation:
\[
\pi_{\alpha(2)}(a_1|\emptyset) = \frac{7}{17} \qquad \pi_{\alpha(2)}(a_2|\emptyset) = \frac{10}{17}
\]

\section{ITERATED BOLTZMANN CONVERGENCE}\label{sec:iterated-boltzmann}
\begin{example}[Boltzmann-coherent mountain race]
    Let us revisit the~\Cref{example:two-cards}. We might compute the Boltzmann-coherent policy by using the construction from~\Cref{def:iterated-coherence}.
    \begin{figure}[h]
        \centering
        \begin{tikzpicture}
          \node (root) at (0,0) {${\emptyset}$};
          \node (R) at (2,1) {$\scalebox{2}{\mountain}$};
          \node (Rc) at (2,-1) {$\scalebox{2}{\forest}$};
          \node (RR) at (4,1.5) {$\scalebox{2}{\gold}$};
          \node (RRc) at (4,0.5) {$\scalebox{2}{\skull}$};
          \node (RcR) at (4,-0.5) {$\scalebox{2}{\silver}$};
          \node (RcRc) at (4,-1.5) {$\scalebox{2}{\silver}$};
        
          \draw[->] (root) -- (R) node [midway, above left] {$\frac{1}{2}$};
          \draw[->] (root) -- (Rc) node [midway, below left] {$\frac{1}{2}$};
          \draw[->] (R) -- (RR) node [midway, above left] {$\frac{1}{2}$};
          \draw[->] (R) -- (RRc) node [midway, below left] {$\frac{1}{2}$};
          \draw[->] (Rc) -- (RcR) node [midway, above left] {$\frac{1}{2}$};
          \draw[->] (Rc) -- (RcRc) node [midway, below left] {$\frac{1}{2}$};
        
          \node at (6,1.5) {$\Prob{R} = \frac{1}{4}$};
          \node at (6,0.5) {$\Prob{R} = \frac{3}{4}$};
          \node at (6,-0.5) {$\Prob{R} = \frac{3}{4}$};
          \node at (6,-1.5) {$\Prob{R} = \frac{3}{4}$};
        \end{tikzpicture}
        \\
        \begin{tikzpicture}
          \node (root) at (0,0) {${\emptyset}$};
          \node (R) at (2,1) {$\scalebox{2}{\mountain}$};
          \node (Rc) at (2,-1) {$\scalebox{2}{\forest}$};
          \node (RR) at (4,1.5) {$\scalebox{2}{\gold}$};
          \node (RRc) at (4,0.5) {$\scalebox{2}{\skull}$};
          \node (RcR) at (4,-0.5) {$\scalebox{2}{\silver}$};
          \node (RcRc) at (4,-1.5) {$\scalebox{2}{\silver}$};
        
          \draw[->] (root) -- (R) node [midway, above left] {$\frac{2}{5}$};
          \draw[->] (root) -- (Rc) node [midway, below left] {$\frac{3}{5}$};
          \draw[->] (R) -- (RR) node [midway, above left] {$1$};
          \draw[->] (R) -- (RRc) node [midway, below left] {$0$};
          \draw[->] (Rc) -- (RcR) node [midway, above left] {$\frac{1}{2}$};
          \draw[->] (Rc) -- (RcRc) node [midway, below left] {$\frac{1}{2}$};
        
          \node at (6,1.5) {$\Prob{R} = \frac{1}{4}$};
          \node at (6,0.5) {$\Prob{R} = \frac{3}{4}$};
          \node at (6,-0.5) {$\Prob{R} = \frac{3}{4}$};
          \node at (6,-1.5) {$\Prob{R} = \frac{3}{4}$};
        \end{tikzpicture}
        \\
        \begin{tikzpicture}
          \node (root) at (0,0) {${\emptyset}$};
          \node (R) at (2,1) {$\scalebox{2}{\mountain}$};
          \node (Rc) at (2,-1) {$\scalebox{2}{\forest}$};
          \node (RR) at (4,1.5) {$\scalebox{2}{\gold}$};
          \node (RRc) at (4,0.5) {$\scalebox{2}{\skull}$};
          \node (RcR) at (4,-0.5) {$\scalebox{2}{\silver}$};
          \node (RcRc) at (4,-1.5) {$\scalebox{2}{\silver}$};
        
          \draw[->] (root) -- (R) node [midway, above left] {$\frac{4}{7}$};
          \draw[->] (root) -- (Rc) node [midway, below left] {$\frac{3}{7}$};
          \draw[->] (R) -- (RR) node [midway, above left] {$1$};
          \draw[->] (R) -- (RRc) node [midway, below left] {$0$};
          \draw[->] (Rc) -- (RcR) node [midway, above left] {$\frac{1}{2}$};
          \draw[->] (Rc) -- (RcRc) node [midway, below left] {$\frac{1}{2}$};
        
          \node at (6,1.5) {$\Prob{R} = \frac{1}{4}$};
          \node at (6,0.5) {$\Prob{R} = \frac{3}{4}$};
          \node at (6,-0.5) {$\Prob{R} = \frac{3}{4}$};
          \node at (6,-1.5) {$\Prob{R} = \frac{3}{4}$};
        \end{tikzpicture}
        \caption{The fixed point of iterated $f$-coherence achieved after $t = 2 = T$ iterations.}
        \label{fig:enter-label}
    \end{figure}
\end{example}
\section{POLICY STABILITY}

In search of the sufficient conditions for the autoregressive goal-conditioned policy to be optimal, one point of focus might be the question of how to \emph{extend} trajectories: we might hope to stitch the optimal trajectory from actions that are optimal at a k-step-lookahead; optimal policies should then be those that behave consistently, in the sense that it would make the same decisions as if it were allowed to take actions looking $k$ steps into the future. Formally, we have the following definition.

\begin{definition}[Policy-stable reasoning]
    Given deterministic dynamics $\tau$ and a prior $p(a_t|s_t)$, we say that a policy $\pi(a_t|s_t) \propto p(\OO_{t:T}|s_t, a_t)$ is \emph{(1-)policy-stable}, if for any actions $a^1_t, a^1_{t+1}, a^2_t, a^2_{t+1}$, with $a^i_{t+1} = \arg\max_{a^i_{t+1}} p(\OO_{t:T}|a^1_{t, t+1}, s_t)$ (where $s^i_{t+1} = \tau(s_t, a^i_t)$), we have that:
    \[
        p(\OO_{t:T}|a^1_{t, t+1}, s_t) > p(\OO_{t:T}|a^2_{t, t+1}, s_t) \implies  
        p(\OO_{t:T}|a^1_{t}, s_t) > p(\OO_{t:T}|a^2_{t}, s_t)
    \]
    In other words: given that that $a^i_{t+1}$  best continues $a^i_t$, if $(a^1_t, a^1_{t+1})$ is preferred to $(a^2_t, a^2_{t+1})$, then $a^1_t$ should be preferred to $a^2_t$.

\end{definition}

\begin{figure}
\centering
    \begin{tikzpicture}
      \node (root) at (0,0) {$s_0$};
      \node (R) at (1.5,1) {$s_1$};
      \node (Rc) at (1.5,-1) {$s'_1$};
      
      \node (RR) at (3,1) {$s_2$};
      \node (RRc) at (3,-1) {$s'_2$};
      
      \node (RcR) at (4.5,1) {$s_3$};
      \node (RcRc) at (4.5,-0.5) {$s'_3$};
      \node (RcRcRc) at (4.5,-1.5) {$s''_3$};
    
      \draw[->] (root) -- (R) node [midway, above left] {$\pi(s_1|s_0)$}; %
      \draw[->] (root) -- (Rc) node [midway, below left] {$\pi(s'_1|s_0)$}; %
      
      \draw[->] (R) -- (RR) node [midway, above ] {$1$};
      \draw[->] (Rc) -- (RRc) node [midway, below ] {$1$};
      \draw[->] (RR) -- (RcR) node [midway, above ] {$1$};
      \draw[->] (RRc) -- (RcRc) node [midway, above left] {$\frac{2}{3}$};
      \draw[->] (RRc) -- (RcRcRc) node [midway, below] {$\frac{1}{3}$};
    
      \node at (6,1) {$\Prob{R} = \frac{1}{2}$};
      \node at (6,-0.5) {$\Prob{R} = 1$};
      \node at (6,-1.5) {$\Prob{R} = 0$};
    \end{tikzpicture}
    \caption{$\pi$ cannot satisfy both 2-policy-stability and 1-policy-stability.}
    \label{fig:not-policy-stable}
\end{figure}

Unfortunately, not only this is not a sufficient condition - in some instances, it is even contradicting optimality. This is because it is impossible to properly extend the the policy-stability over an arbitrary number of actions. To show that, we can introduce a notion of $n$-policy-stable predictor, which says that for any sequences $a_{1:k}, a'_{1:k}$ and actions $a, a'$, such that the sequence $a_{1:k}$ best continues $a$ and sequence $a'_{1:k}$ best continues $a'$, we have that $(a, a_{1:k})$ being preferred over $(a', a'_{1:k})$ implies that $a$ is preferred to $a'$. 

From direct calculation, it can be seen that, if a policy derived by control-as-inference from a prior over the MDP shown in the \Cref{fig:not-policy-stable} is to be 1-policy-stable, then it must be the case that $\pi(s_1|s_0) \geq \pi(s'_1|s_0)$, while considering two actions into the future and then following the prior, moving into $s_2$ is a better choice. However, to be 2-policy-stable, it has to satisfy $\pi(s_1|s_0) < \pi(s'_1|s_0)$, since after considering three moves in the future, moving to $s'_1$ becomes the better choice. Since 1-policy-stable and 2-policy-stable are mutually exclusive for a policy derived by control-as-inference from this prior, one has to make a choice as to which $n$-policy-stability to require, which, in practice, requires the knowledge of the time horizon (and for it to be fixed and finite).

\section{EXPERIMENTAL DETAILS}\label{app:experiments}

\subsection{MDP environments}

We perform all experiments on a suite of 140 small randomly generated tree MDPs of varying sizes. We set the number of actions $|A|$ between 2 and 4, and time horizon $T$ between 3 and 6, resulting in environments ranging in state space sizes between $4$ and $122$, and total state-action sizes between $8$ and $366$. For the misalignment experiment, we used both deterministic and stochastic environments in equal proportions, with transition probabilities and rewards sampled uniformly from relevant distributions. For the estimated effective horizon experiment, we used only deterministic environments, because otherwise the estimator was difficult to compute. Full details of the environments are available in the released code repository available at~\url{https://github.com/jkarwowski/incoherence}, in the config file for each experiment. 

\subsection{Estimating effective horizon}\label{app:eff-horizon}

Effective horizon is a metric that is difficult to compute exactly even for small environments. We approximate $\hat{H}$ using Monte Carlo simulation and , to output $\min_{k} H_k$, where $H_k = k + \log_A m$ for which $\text{GORP}(k,m)$ recovers the optimal policy with probability at least 1/2, as in ~\citet[Definition 5.2]{laidlaw_bridging_2024}. We note this is a different estimator from the one used by the original paper - see Appendix C and \href{https://github.com/cassidylaidlaw/effective-horizon}{their codebase}. This is because they work on much larger environments such as Atari games, for which our approach would be impractical. On the other hand, their upper bounds such as those using \citet[Theorem 5.4]{laidlaw_bridging_2024} turned to be insufficiently tight for our purposes.

\subsection{Misalignment versus incoherence}

The correlation between incoherence and misalignment depends on the specific value of temperature $\delta$. We present the graph of correlations in~\Cref{fig:misalignment-incoherence-global}. For small values of $\delta$, the policy is converges to deterministic, making incoherence very large variance and correlations disappear. On the other hand, in the limit of $\delta \to \infty$, the policy tends to uniform distribution and again rendering incoherence useless. We note that this is only a heuristic explanation, note and we do not have a satisfactory theoretical explanation for the exact shape of the graph.

\begin{figure*}
\begin{subfigure}{0.5\textwidth}
    \includegraphics[width=0.9\linewidth]{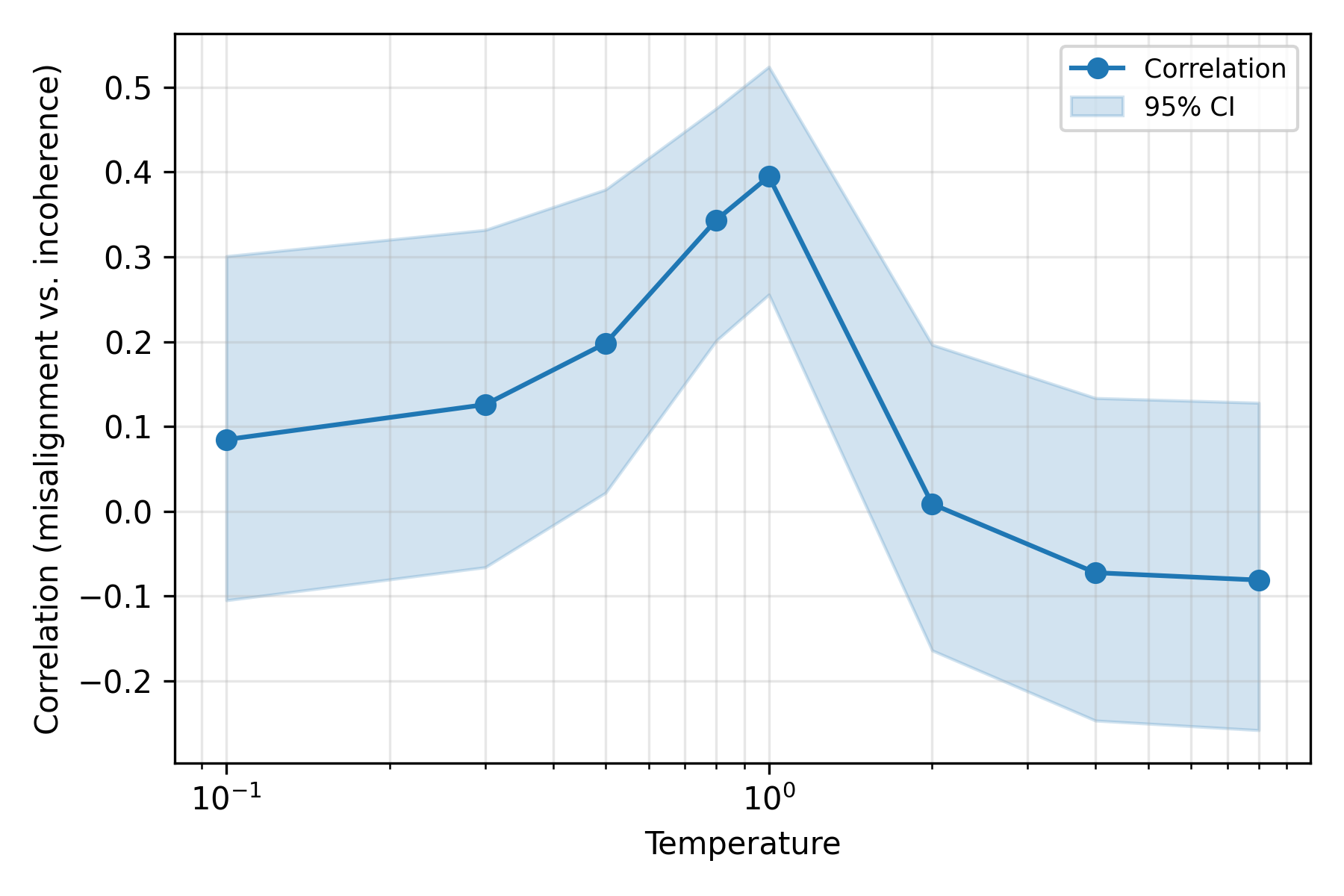}
\caption{\phantom{a}}
\label{fig:misalignment-incoherence-global}
\end{subfigure}
\begin{subfigure}{0.5\textwidth}
    \includegraphics[width=0.9\linewidth]{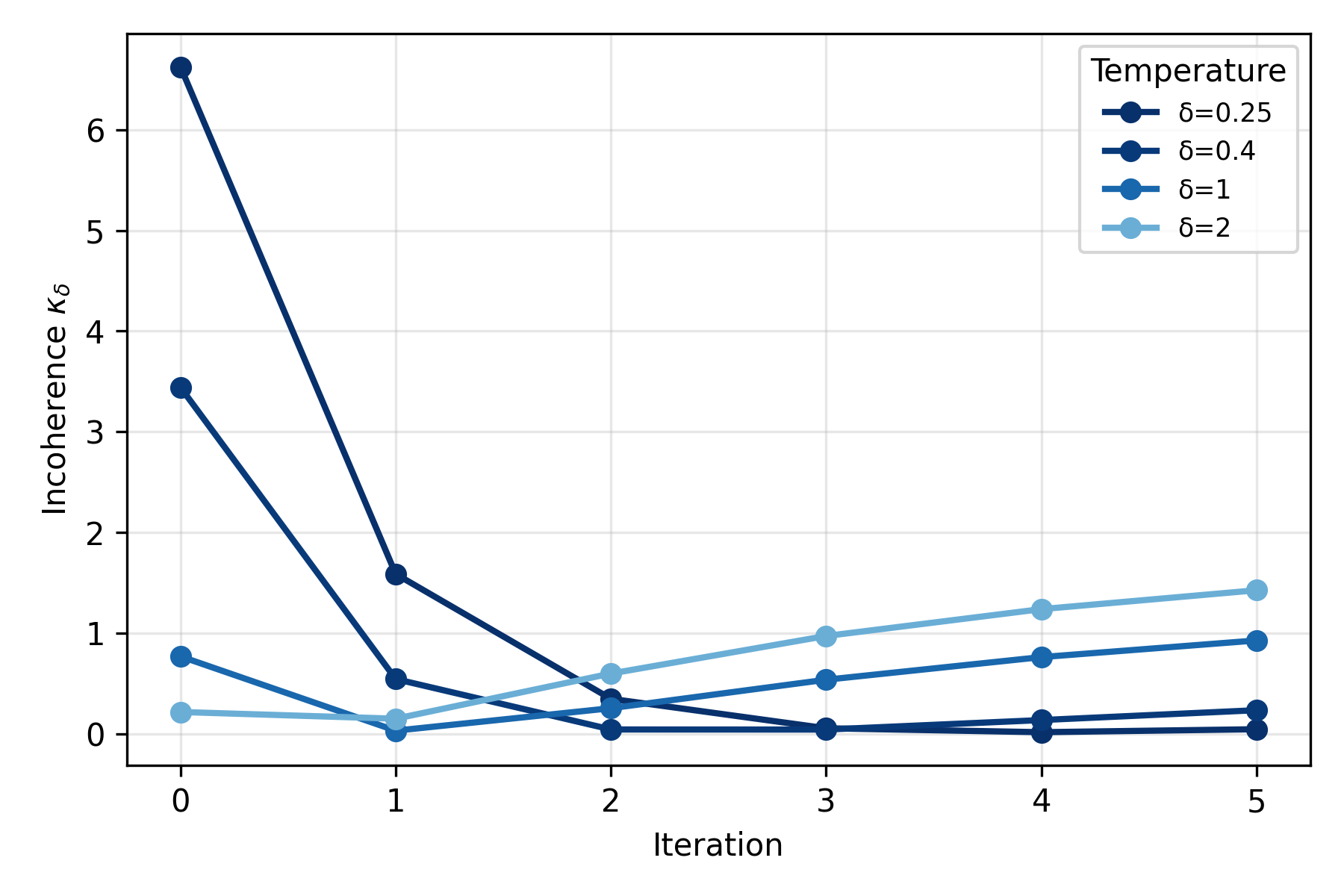}
\caption{\phantom{b}}
\label{fig:cor-5-11}
\end{subfigure}

\caption{(a) An aggregated version of the graph in~\Cref{fig:misalignment-incoherence-temp1}, each point corresponding to a particular choice of temperature. Correlation between misalignment and incoherence is only meaningful in moderate temperatures near $\delta=1$. (b) Experimental illustration of~\Cref{corr:limit}, only in the simultaneous limit of $(\delta, i)\to(0, \infty)$ incoherence vanishes.}
\label{fig:image2}
\end{figure*}

\subsection{Experimental validation of~\Cref{corr:limit} and~\Cref{prop:strong-improvement}}
\Cref{corr:limit} states that in the limit of both $(\delta, i) \to (0, \infty)$, incoherence $\kappa_\delta(\pi^i)$ vanishes. We note that taking both limits simultaneously is essential: indeed, holding one coordinate fixed, we do not get convergence to $0$.
We illustrate this point experimentally in~\Cref{fig:cor-5-11}. Each line shows incoherence for a different choice of temperature $\delta$. Holding $\delta$ fixed and increasing iteration $i$, the incoherence of $\pi^i$ initially decreases, such that at a certain time it dips below every higher temperature, but later starts to increase. This is because~\Cref{prop:strong-improvement} guarantees the improvement in return, which is not equal to the KL-regularised return optimised for by the soft-conditioned policies.

\Cref{prop:strong-improvement} states that regardless of the environment, we should see a consistent improvement in return. We validate this empirically for 5 randomly chosen deterministic-dynamics MDPs and 5 randomly chosen stochastic-dynamics MDPs, $|A| = 2$, $|T| = 4$, both with $i = 5$ iteration steps, with results conforming to the theory.

\begin{figure*}
\begin{subfigure}{0.5\textwidth}
    \includegraphics[width=0.9\linewidth]{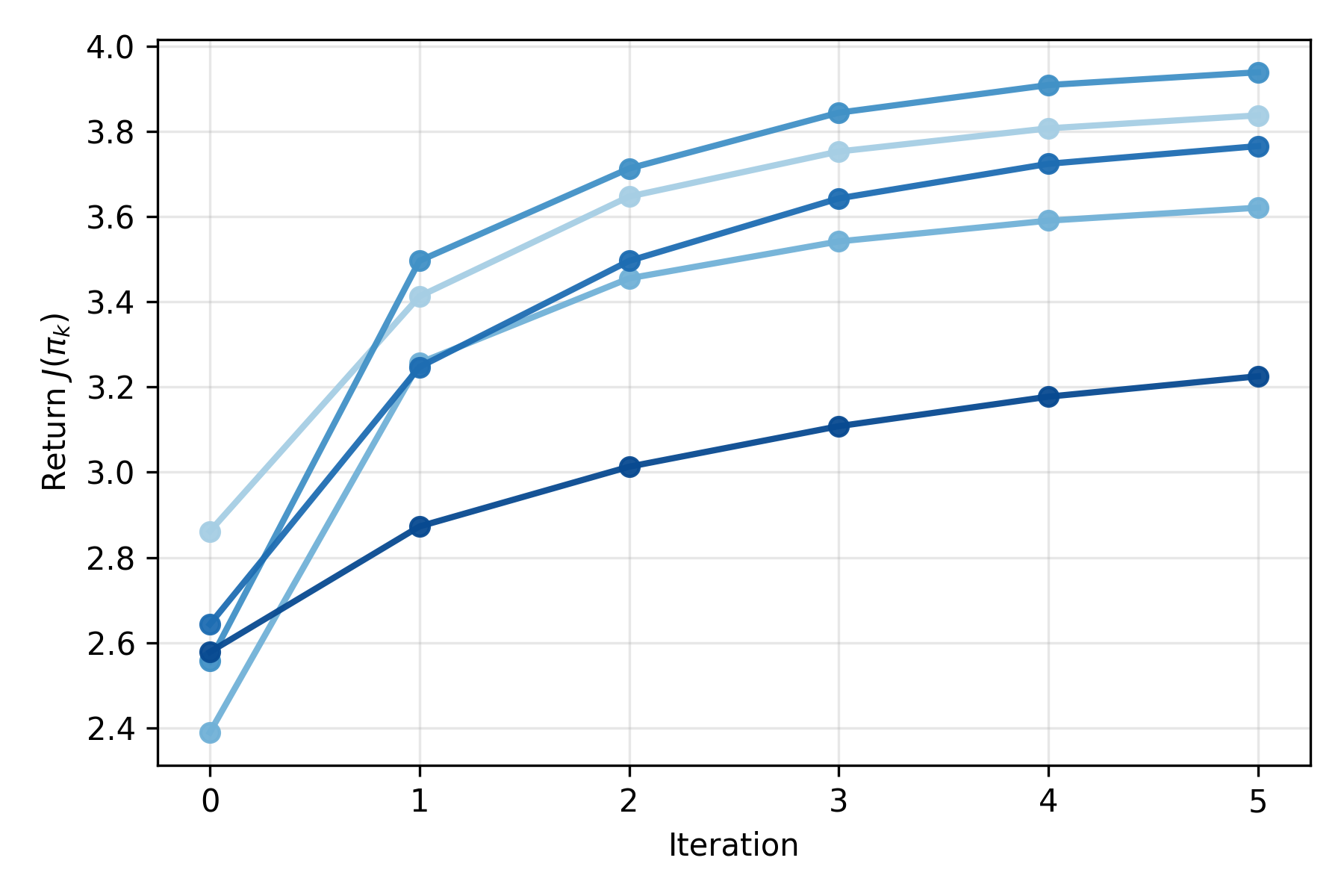}
\caption{\phantom{a}}
\label{fig:subim1}
\end{subfigure}
\begin{subfigure}{0.5\textwidth}
    \includegraphics[width=0.9\linewidth]{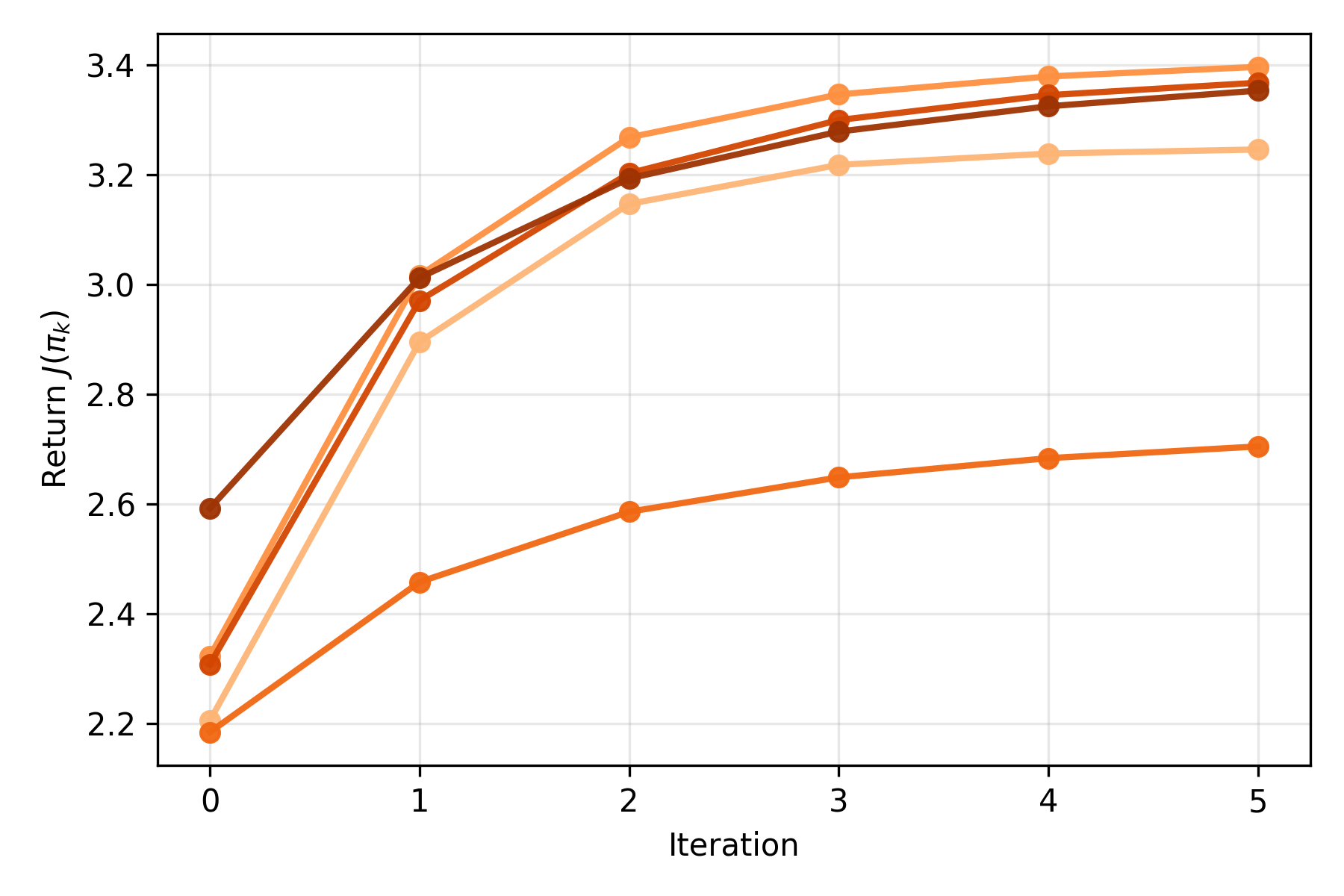}
\caption{\phantom{b}}
\label{fig:subim2}
\end{subfigure}

\caption{Experimental validation of~\Cref{corr:limit}, (a) on $5$ randomly chosen deterministic MDPs, and (b) on 5 randomly chosen stochastic environments.}
\label{fig:image2}
\end{figure*}

\subsection{Confidence intervals}

All confidence intervals in the paper are constructed with the bootstrap method, using $B = 2000$ samples.

\end{document}